\documentclass{colt2015} 

\usepackage{algorithm,algorithmic}
\usepackage{amssymb, multirow, paralist, color}
\usepackage{enumitem}
\usepackage{textcase}
\usepackage[T1]{fontenc}
\usepackage{caption}
\newtheorem{thm}{Theorem}
\newtheorem{prop}{Proposition}
\newtheorem{lem}{Lemma}
\newtheorem{cor}[thm]{Corollary}
\newtheorem{defn}[thm]{Definition}
\newtheorem{ass}[thm]{Assumption}
\usepackage{enumitem}

\def \y {\mathbf{y}}
\def \E {\mathbb{E}}
\def \x {\mathbf{x}}

\def \z {\mathbf{z}}
\def \u {\mathbf{u}}
\def \H {\mathcal{H}}
\def \w {\mathbf{w}}
\def \R {\mathbb{R}}

\def \Z {\mathcal{Z}}

\def \W {\mathcal{W}}
\def \N {\mathcal{N}}

\def \v {\mathbf{v}}

\def \b {\mathbf{b}}

\def \B {\mathcal{B}}

\def \wt {\widetilde{\w}}

\def \X {\mathcal{X}}

\def \P {\mathbb{P}}

\def \wh {\widehat{\w}}

\def \B {\mathcal{B}}

\DeclareMathOperator*{\argmin}{arg\,min}

\begin{document}

\title[Fast Rates of ERM and Stochastic Approximation]{Fast Rates of ERM and Stochastic Approximation:\\ Adaptive to Error Bound Conditions}

 \coltauthor{\Name{Mingrui Liu}$^\dagger$\Email{mingrui-liu@uiowa.edu}\\
	\Name{Xiaoxuan Zhang}$^\dagger$ \Email{xiaoxuan-zhang@uiowa.edu}\\
	\Name{Lijun Zhang}$^\ddagger$\Email{zljzju@gmail.com}\\
	\Name{Rong Jin}$^*$\Email{jinrong.jr@alibaba-inc.com}\\
	\Name{Tianbao Yang}$^\dagger$\Email{tianbao-yang@uiowa.edu}\\
	\addr$^\dagger$Department of Computer Science, The University of Iowa, Iowa City, IA 52242, USA \\
	\addr$\ddagger$National Key Laboratory for Novel Software Technology, Nanjing University, Nanjing, China\\
	\addr$^*$Alibaba Group, Seattle, USA\\
}
\maketitle
\vspace*{-0.5in}
\begin{center}{First version: February  10, 2018~\footnote{This is the date this version was circulated.}}\end{center}
\vspace*{0.2in}

\begin{abstract}
	Error bound conditions (EBC) are properties that  characterize the growth of an objective function when a point is moved away from the optimal set. They have  recently received increasing attention in the field  of optimization for developing optimization algorithms with fast convergence.  However,  the studies  of EBC in statistical learning are hitherto still limited.   The main contributions of this paper are two-fold. First,  we develop fast and intermediate rates of  empirical risk minimization (ERM) under EBC for risk minimization with Lipschitz continuous, and  smooth  convex random functions. Second, we establish fast and intermediate rates of an efficient stochastic approximation (SA) algorithm for risk minimization  with Lipschitz continuous random functions, which requires only one pass of $n$ samples and adapts to  EBC. For both approaches, the convergence rates span a full spectrum between $\widetilde O(1/\sqrt{n})$ and $\widetilde O(1/n)$ depending on the power constant in EBC, and could be even faster than $O(1/n)$ in special cases for ERM. Moreover, these  convergence rates are automatically adaptive without using any knowledge of EBC. Overall, this work not only strengthens the understanding of ERM for statistical learning but also brings  new fast stochastic algorithms for solving a broad range of statistical learning problems. 
	
\end{abstract}

\section{Introduction}
In this paper, we focus on the following stochastic convex optimization problems arising in statistical learning and many other fields: 
\begin{align}\label{eqn:opt}
	\min_{\w\in\W} P(\w) \triangleq \E_{\z\sim \P}[f(\w, \z)]
\end{align}
and more generally 
\begin{align}\label{eqn:opt2}
	\min_{\w\in\W} P(\w) \triangleq \E_{\z\sim \P}[f(\w, \z)] + r(\w)
\end{align}
where $f(\cdot, \z): \W\rightarrow\R$ is a random function depending on a random variable $\z\in\Z$ that follows a distribution $\P$, $r(\w)$ is a lower semi-continuous convex function. In statistical learning~\citep{Vapnik1998}, the problem above is also referred to as {\bf risk minimization} where $\z$ is interpreted as data, $\w$  is interpreted as a model (or hypothesis), $f(\cdot, \cdot)$ is interpreted as a loss function, and $r(\cdot)$ is a regularization. For example, in supervised learning one can take $\z=(\x, y)$ - a pair of feature vector $\x\in\X\subseteq\R^d$ and label $\y\in\mathcal Y$, $f(\w, \z) = \ell(\w(\x), y)$ - a loss function measuring the error of the prediction $\w(\x):\mathcal X\rightarrow \mathcal Y$ made by the model $\w$. Nonetheless, we emphasize that the risk minimization problem~(\ref{eqn:opt}) is more general than supervised learning and could be more challenging (c.f.~\citep{COLT:Shalev:2009}).  In this paper, we assume that $\W\subseteq\R^d$ is a compact and convex set. Let $\W_* = \arg\min_{\w\in\W}P(\w)$ denote the optimal set and $P_* = \min_{\w\in\W}P(\w)$ denote the optimal risk. 

There are two popular approaches for solving the risk minimization problem. The first one is by empirical risk minimization that minimizes the empirical risk defined over a set of $n$ i.i.d. samples drawn from the same distribution $\P$ (sometimes with a regularization term on the model). The second approach is called stochastic approximation that iteratively learns the model from random samples $\z_t\sim\P, t=1,\ldots, n$. Both approaches have been studied broadly and extensive results are available about the theoretical guarantee of  the two approaches in the machine learning and optimization community. A central theme in these studies is to bound the excess risk (or optimization error)  of a learned model $\wh$ measured by   $P(\wh) - P_*$, i.e.,   given a set of $n$ samples $(\z_1,\ldots, \z_n)$ how fast the learned model converges to the optimal model in terms of the excess risk.

A classical result about the excess risk bound  for the considered risk minimization problem is in the order of $\widetilde O(\sqrt{d/n})$~\footnote{$\widetilde O$ hides a poly-logarithmic factor of $n$.} and  $O(\sqrt{1/n})$ for ERM and SA, respectively, under appropriate conditions of the loss functions (e.g., Lipschitz continuity,  convexity)~\citep{Nemirovski:2009:RSA:1654243.1654247,COLT:Shalev:2009}. Various studies have attempted to establish faster rates by imposing additional conditions on the loss functions (e.g., strong convexity, smoothness, exponential concavity), or on both the loss functions and the distribution (e.g., Tsybakov condition, Bernstein condition, central condition). In this paper, we will study a different family of conditions called the error bound conditions  (EBC) (see Definition~\ref{def:1}), which has a long history in the community of optimization  and variational analysis~\citep{DBLP:journals/mp/Pang97} and recently revives for developing fast optimization algorithms without strong convexity~\citep{arxiv:1510.08234,Drusvyatskiy16a,DBLP:conf/pkdd/KarimiNS16,DBLP:journals/corr/nesterov16linearnon,DBLP:journals/corr/abs/1606.00269}. However, the exploration of EBC in statistical learning for risk minimization is still under-explored and the connection to other conditions is not fully understood. 
\begin{defn}\label{def:1}
	For any $\w\in\W$, let $\w^*=\arg\min_{\u\in\W_*}\|\u - \w\|_2$ denote an optimal solution closest to $\w$. Let $\theta\in(0, 1]$ and $0<\alpha<\infty$.  The problem~(\ref{eqn:opt}) satisfies an $\text{EBC}(\theta, \alpha)$ if for any $\w\in\W$, the following inequality holds
	\begin{align}\label{eqn:EBC}
		\|\w - \w^*\|_2^2\leq \alpha(P(\w) - P(\w^*))^\theta
	\end{align}
\end{defn}
This condition has been well studied in optimization and variational analysis. Many results are available for understanding the condition for different problems. For example,  it has  been shown 
that when $P(\w)$ is semi-algebraic and continuous, the inequality~(\ref{eqn:EBC}) is known to hold on any compact set with certain $\theta\in(0,1]$ and $\alpha>0$~\citep{arxiv:1510.08234}~\footnote{In related literature, one may also consider $\theta\in(1,2]$, which will yield the same order of excess risk bound as $\theta=1$ in our settings.  }.
We will study both ERM and SA under the above error bound condition. In particular, we show that the benefits of exploiting EBC in statistical learning are noticeable and profound by establishing the following results. 
\begin{itemize}[leftmargin=*]
	\item {\bf Result I.}  First, we show that for Lipchitz continuos loss EBC implies a {\it relaxed}  Bernstein condition, and therefore leads to intermediate  rates of $\widetilde O\left(\left(\frac{d}{n}\right)^{\frac{1}{2-\theta}}\right)$ for Lipschitz continuous loss. 
Although this result does not improve over existing rates based on Bernstein condition, however, we emphasize that it provides an alternative route for establishing fast rates and brings richer results than literature to statistical learning in light of the examples provided  in this paper.

	\item {\bf Result II.}  Second, we develop fast and optimistic rates of ERM for non-negative, Lipschitz continuous and smooth convex loss functions in the order of $\widetilde O\left(\frac{d}{n} + \left(\frac{d P_*}{n}\right)^{\frac{1}{2-\theta}}\right)$,
	and in the order of $\widetilde O\left(\left(\frac{d}{n}\right)^{\frac{2}{2-\theta}} + \left(\frac{d P_*}{n}\right)^{\frac{1}{2-\theta}}\right)$ when the sample size $n$ is sufficiently large, which imply that when the optimal risk $P_*$ is small one can achieve a fast rate of $\widetilde O\left(\frac{d}{n}\right)$ even with $\theta < 1$ and a faster rate of $\widetilde O\left(\left(\frac{d}{n}\right)^{\frac{2}{2-\theta}}\right)$ when $n$ is sufficiently large. 
	\item {\bf Result III.} Third, we develop an efficient SA algorithm with almost the same per-iteration cost as stochastic subgradient methods for Lipschitz continuous loss, which achieves the same order of rate $\widetilde O\left(\left(\frac{1}{n}\right)^{\frac{1}{2-\theta}}\right)$ as ERM without an explicit dependence on $d$. More importantly it is ``parameter''-free with no need of prior knowledge of $\theta$ and $\alpha$ in EBC. 
\end{itemize}

Before ending this section, we would like to point out  that all the results 
are adaptive to the largest possible value of $\theta\in(0,1]$ in hindsight  of the problem, and the dependence on $d$ for ERM is generally unavoidable according to the lower bounds studied in~\citep{NIPS2016_ERM}.

\section{Related Work}
In this section, we review some related work to better understand our established results. We note that there are extensive work about the analysis of generalization ability of ERM and SA, especially on showing the classical $O(1/\sqrt{n})$ rate. Instead of being exhaustive, here we focus on closely related studies about fast rates or intermediate rates of ERM and SA. 

The results for statistical learning  under EBC are limited. A similar one to our {\bf Result I} for ERM was established in~\citep{Shapiro:2014:LSP:2678054}. However, their result requires the convexity condition of random loss functions, making it weaker than our result. \citet{DBLP:conf/icml/RamdasS13} and \citet{ICMLASSG} considered SA under the EBC condition and established similar adaptive rates. Nonetheless, their stochastic algorithms require knowing the values of $\theta$ and possibly the constant $\alpha$ in the EBC. In contrast, the SA algorithm in this paper is ``parameter"-free without the need of knowing $\theta$ and $\alpha$ while still achieving the adaptive rates of $O(1/n^{2-\theta})$. 

Fast rates under strong convexity are  well-known  for ERM,  online optimization and stochastic optimization~\citep{COLT:Shalev:2009,DBLP:conf/nips/SridharanSS08,DBLP:journals/ml/HazanAK07,DBLP:conf/nips/KakadeT08,shalev-shwartz-2007-pegasos,hazan-20110-beyond}. 
A weaker condition than strong convexity, namely exponential concavity (exp-concavity), has also attracted significant attention for developing fast rates in online learning and statistical learning. Studies have explored exp-concavity in online learning 
and have achieved an $O(\log (n))$ regret bound for $n$ rounds~\citep{Vovk:1990:AS:92571.92672,DBLP:journals/ml/HazanAK07}. Several recent works established  the $\widetilde O(d/n)$ fast rate of ERM with exp-concave loss functions over a bounded domain $\W\subseteq\R^d$~\citep{DBLP:journals/corr/abs/1601.04011,DBLP:conf/nips/KorenL15,arXiv:1605.01288}. 

The Bernstein condition (see Definition~\ref{def:2}), itself a generalization of Tsybakov margin condition for classification, was introduced in~\citep{Bartlett2006} and played an important role for developing fast and intermediate excess risk bounds in many works~\citep{Local_RC,Local:Vladimir}. Recently, a different family of conditions (named stochastic mixability condition or the $v$-central condition (see Definition~\ref{def:3})) was introduced for developing fast and intermediate rates of ERM~\citep{DBLP:journals/jmlr/ErvenGMRW15}. The connection between the exp-concavity condition, the Bernstein condition and the $v$-central condition was studied in~\citep{DBLP:journals/jmlr/ErvenGMRW15}. In particular, the exp-concavity implies a $v$-central condition under an appropriate condition of the decision set $\W$ (e.g., well-specificity or convexity). With the bounded loss condition, the Bernstein condition implies the $v$-central condition and the $v$-central condition also implies a Bernstein condition. 

In this work, we also study the connection between the EBC and the Bernstein condition and the $v$-central condition. In particular, we will develop weaker forms of the Bernstein condition and the $v$-central condition from the EBC for Lipschitz continuous loss functions. Building on this connection, we establish our {\bf Result I}, which is on a par with existing results for bounded loss functions relying on the Bernstein condition or the central condition.  
Nevertheless, we emphasize that employing the EBC for developing fast rates has noticeable benefits: (i) it is complementary to the Bernstein condition and the central condition and enjoyed by several interesting problems whose fast rates are not exhibited yet;   (ii) it can be leveraged for developing fast and intermediate optimistic rates  for non-negative and smooth loss functions; (iii) it can be leveraged to develop efficient SA algorithms with intermediate and fast convergence rates.

\citet{DBLP:conf/nips/SrebroST10} established an optimistic rate of $O\left(1/n + \sqrt{P_*/n}\right)$ of both ERM and SA for supervised learning with generalized linear loss functions. However, their SA algorithm requires knowing the value of $P_*$. 
Recently, \citet{DBLP:journals/corr/0005YJ17} considered the general stochastic optimization problem~(\ref{eqn:opt}) with non-negative and smooth loss functions and achieved a series of optimistic results. It is worth  mentioning that their excess risk bounds for both convex problems and strongly convex problems are  special cases of our {\bf Result II} when $\theta=0$ and $\theta=1$, respectively. However, the intermediate optimistic rates for $\theta\in(0,1)$ are first shown in this paper. Importantly, our {\bf Result II} under the EBC with $\theta=1$ is more general than the result in~\citep{DBLP:journals/corr/0005YJ17} under strong convexity assumption. 

Finally, we discuss about stochastic approximation algorithms with fast and intermediate rates to understand the significance of our {\bf Result III}. Different variants of stochastic gradient methods have been analyzed for  stochastic strongly convex optimization~\citep{hazan-20110-beyond,ICML2012Rakhlin,ICML2013Shamir:ICML} with a fast rate of $O(1/n)$. 
But these stochastic algorithms require knowing the strong convexity modulus.
A recent work established  adaptive regret bounds $O(n^{\frac{1-\theta}{2-\theta}})$ for online learning with a total of $n$ rounds under the Bernstein condition~\citep{DBLP:conf/nips/KoolenGE16}. However, their methods are based on the second-order methods and therefore are not as efficient as our stochastic approximation algorithm. For example, for online convex optimization they employed  the MetaGrad algorithm~\citep{DBLP:conf/nips/ErvenK16}, which needs to maintain $\log (n)$ copies of  the online Newton step (ONS)~\citep{DBLP:journals/ml/HazanAK07} with different learning rates. Notice that the per-iteration cost of ONS is usually $O(d^4)$ even for very simple domain $\W$~\citep{DBLP:conf/nips/KorenL15}, while that of our SA algorithm is dominated by the Euclidean projection onto $\W$ that is as fast as $O(d)$ for a simple domain.

\section{Empirical Risk Minimization (ERM)}
We first formally state the minimal assumptions that are made  throughout the paper. Additional assumptions will be made in the sequel  for developing fast rates for different families of the random functions $f(\w,\z)$. 
\begin{ass}\label{ass:1}
	For the stochastic optimization problems~(\ref{eqn:opt}) and~(\ref{eqn:opt2}), we assume: 
		(i) $P(\w)$ is a convex function,  $\W$ is a closed and bounded convex set, i.e., there exists $R>0$ such that $\|\w\|_2\leq R$ for any $\w\in\W$, and $r(\w)$ is a Lipschitz continuous convex function. 
		(ii) the  problem~(\ref{eqn:opt}) and~(\ref{eqn:opt2}) satisfy an EBC$(\theta, \alpha)$, i.e., there exist $\theta\in(0,1]$ and $0<\alpha<\infty$ such that the inequality~(\ref{eqn:EBC}) hold. 
\end{ass}

In this section, we focus on the development of theory of ERM for risk minimization. In particular, we learn a model $\wh$ by solving the following ERM problem corresponding to~(\ref{eqn:opt}): 
\begin{align}
	\wh \in \arg\min_{\w\in\W}P_n(\w)\triangleq  \frac{1}{n}\sum_{i=1}^nf(\w, \z_i)
\end{align}
where $\z_1,\ldots, \z_n$ are i.i.d samples following the distribution $\P$.  A similar ERM problem can be formulated for~(\ref{eqn:opt2}). This section is divided into two subsections. In the first subsection, we establish intermediate rates of ERM under EBC when the random function is Lipschitz continuous. In the second subsection, we develop intermediate rates of ERM under EBC when the random function is smooth. In the sequel and the supplement, we use $\vee$ to denote the max operation and use $\wedge$ to denote the min operation.

\subsection{ERM for Lipschitz continuous random functions}
In this subsection, w.l.o.g we restrict our attention to~(\ref{eqn:opt}) since we  make the following assumption besides {Assumption~\ref{ass:1}}. 
\begin{ass}\label{ass:2}
	For the stochastic optimization problem~(\ref{eqn:opt}), we assume that $f(\w, \z)$ is a Lipschitz continuous function w.r.t $\w$ for any $\z\in\Z$, i.e., there exists $G>0$ such that for any $\w, \u\in\W$,
	\[
	|f(\w, \z) - f(\u, \z)|\leq G\|\w - \u\|_2, \forall \z\in\Z.
	\] 
\end{ass}
If $g(\w)$ is present,  it can be absorbed into $f(\w, \z)$. It is notable that we do not assume $f(\w, \z)$ is convex in terms of $\w$ or any $\z$.

First, we compare EBC with two very important conditions considered in literature for developing fast rates of ERM,  namely the Bernstein condition and the central condition. 
We first give the definitions of these two conditions. 
\begin{defn}(Bernstein Condition)\label{def:2}
	Let $\beta\in(0,1]$ and $B\geq 1$.  Then $(f, \P, \W)$  satisfies the $(\beta, B)$-Bernstein condition if there exists a $\w_*\in \W$ such that for any $\w\in\W$
	\begin{align}\label{eqn:Bern}
		\E_{\z}[(f(\w, \z) - f(\w_*, \z))^2]\leq B(\E_\z[f(\w, \z) - f(\w_*, \z)])^\beta.
	\end{align}
\end{defn}
It is clear that if such an $\w_*$ exists it has to be the minimizer of the risk. 
\begin{defn}($v$-Central Condition)\label{def:3}
	Let $v:[0, \infty)\rightarrow [0, \infty)$  be a bounded, non-decreasing function satisfying $v(x)>0$ for all $x>0$. We say that  $(f, \P, \W)$ satisfies the $v$-central condition if for all $\varepsilon\geq 0$, there exists $\w_*\in \W$ such that for any $\w\in\W$
	\begin{align}\label{eqn:vc}
		\E_{\z\sim \P}\left[e^{\eta(f(\w_*, \z) - f(\w, \z))}\right]\leq e^{\eta\varepsilon}
	\end{align}
	holds with $\eta = v(\varepsilon)$.
\end{defn}
If $v(\varepsilon)$ is a constant for all $\varepsilon\geq 0$, the $v$-central condition reduces to the strong $\eta$-central condition, which implies the $O(1/n)$ fast rate~\citep{DBLP:journals/jmlr/ErvenGMRW15}.  The connection  between the Bernstein condition or $v$-central condition has been studied in~\citep{DBLP:journals/jmlr/ErvenGMRW15}. For example, if the random functions  $f(\w, \z)$ take values in $[0, a]$, then $(\beta, B)$-Bernstein condition implies $v$-central condition with $v(x)\propto x^{1-\beta}$.

The following lemma shows that for Lipchitz continuous function, EBC condition implies a relaxed Bernstein condition and a relaxed $v$-central condition. 
\begin{lem}\label{lem:1}({\bf Relaxed Bernstein condition and $v$-central condition})
	Suppose {Assumptions~\ref{ass:1},~\ref{ass:2}} hold. For any $\w\in\W$,  there exists  $\w^*\in\W_*$ (which is actually the one closest to $\w$), such that 
	\begin{align*}
&\E_{\z}[(f(\w, \z) - f(\w^*, \z))^2] \leq B(\E_\z[f(\w, \z) - f(\w^*, \z)])^\theta,
\end{align*}
where $B=G^2\alpha$, and 
\begin{align*}
	\E_{\z\sim \P}\left[e^{\eta(f(\w^*, \z) - f(\w, \z))}\right]\leq e^{\eta\varepsilon},
\end{align*} 	
	where $\eta = v(\varepsilon):=c\varepsilon^{1-\theta}\wedge b$. Additionally, for any $\varepsilon>0$ if $P(\w) - P(\w^*)\geq \varepsilon$, we have 
	\[
	\E_{\z\sim \P}\left[e^{v(\varepsilon)(f(\w^*, \z) - f(\w, \z))}\right]\leq 1
	\]
	where $b>0$  is any constant and $c= 1/(\alpha G^2\kappa(4GRb))$, where $\kappa(x) = (e^x - x - 1)/x^2$. 
\end{lem}

{\bf Remark:} There is a subtle difference between the above relaxed Bernstein condition and $v$-central condition and their original definitions in Definitions~\ref{def:2} and~\ref{def:3}. The difference is that in Definitions~\ref{def:2} and~\ref{def:3}, it requires there exists a universal  $\w_*$ for all $\w\in\W$ such that ~(\ref{eqn:Bern}) and~(\ref{eqn:vc}) hold.  In Lemma~\ref{lem:1} it only requires for every $\w\in\W$ there exists one $\w^*$ that could be different for different $\w$  such that ~(\ref{eqn:Bern}) and~(\ref{eqn:vc}) hold. This relaxation enables us to establish richer results by exploring EBC than the  Bernstein condition and $v$-central condition, which are postponed to Section~\ref{sec:exam}. 


In addition to the difference highlighted above, we would like to point out that EBC is complementary to  the Bernstein or the $v$-central condition. 
In particular, we use two examples given in~\citep{DBLP:journals/jmlr/ErvenGMRW15} to show that EBC holds but the Bernstein condition or the $v$-central condition fails. 

{\bf Example 1.} Consider the square loss $f(w, z)= \frac{1}{2}(w-z)^2$ with $w\in\W=[-1, 1]$. Let $\P$ be a distribution over $z$ such that $\E[z]=0$ and, for some $c_1, c_2>0$, for all $z\in\R$ with $|z|>c_1$, the density $p(z)$ of $\P$ satisfies $p(z)\geq c_2/z^6$. It was shown that the $v$-central conditional fails. But, it is easy to see that EBC$(\theta=1, \alpha)$ is satisfied. 

{\bf Example 2.} Consider the square loss $f(w, z) = \frac{1}{2}(w - z)^2$ with $w\in\W=\R$. Assume $z$ follows a normal distribution with mean $v$ and standard deviation $1$. For all $B\geq 1$, the $(1, B)$-Bernstein condition will fail for $|w|>\sqrt{32B}$. Nevertheless, EBC$(\theta=1, \alpha=2)$ holds.


Next, we present the main result of this subsection. 
\begin{thm}[{\bf Result I}]\label{thm:1}
	Suppose {Assumptions~\ref{ass:1},~\ref{ass:2}} hold. 
	For any $n\geq aC$, with probability at least $1-\delta$ we have
	\begin{align}
		P(\wh) - P_*  \leq O\left(\frac{d\log n  + \log(1/\delta)}{n}\right)^{\frac{1}{2-\theta}}
	\end{align}
	where $a =3 (d\log(32GRn^{1/(2-\theta)}) + \log(1/\delta))/c + 1 $ and  $C>0$ is  some constant.
\end{thm}
{\bf Remark:} The proof utilizes Lemma~\ref{lem:1} and follows similarly as the proofs in previous studies~\citep{DBLP:journals/jmlr/ErvenGMRW15,arXiv:1605.01288} based on $v$-central condition. Our analysis essentially shows that relaxed Bernstein condition and relaxed $v$-central condition with non-universal $\w^*$ suffice to establish the intermediate rates. Although the rate in Theorem~\ref{thm:1} does not improve that in previous works~\citep{DBLP:journals/jmlr/ErvenGMRW15}, the relaxation brought by EBC allows us to establish fast rates for interesting problems that are unknown before.  More details are postponed into Section~\ref{sec:exam}. 
For example, under the condition that the input data $\x, y$ are bounded, ERM for hinge loss minimization with $\ell_1$,  $\ell_\infty$ norm constraints, and for minimizing a quadratic function and $\ell_1$ norm regularization  enjoys an $\widetilde O(1/n)$ fast rate. To the best of our knowledge, such a fast rate of ERM for these problems has not been shown in literature using other conditions or theories.

\subsection{ERM for non-negative, Lipschitz continuous and smooth convex random functions}
In this subsection, we will present improved optimistic rates of ERM for non-negative smooth loss functions expanding the results in~\citep{DBLP:journals/corr/0005YJ17}. To be general, we consider~(\ref{eqn:opt2}) and the following ERM problem:
\begin{align}
	\wh \in \arg\min_{\w\in\W}P_n(\w)\triangleq  \frac{1}{n}\sum_{i=1}^nf(\w, \z_i) + r(\w)
\end{align}
Besides { Assumptions~\ref{ass:1},~\ref{ass:2}}, we further make the following assumption for developing faster rates. 
\begin{ass}\label{ass:3}
	For the stochastic optimization problem~(\ref{eqn:opt}), we assume 
	$f(\w, \z)$ is a non-negative and smooth convex function w.r.t $\w$ for any $\z\in\Z$, i.e., there exists $L\geq 0$ such that for any $\w, \u\in\W$,
	\begin{align*}
	0\leq& f(\w, \z) - f(\u, \z) - \nabla f(\u, \z)^{\top}(\w - \u)
	\leq \frac{L}{2}\|\w - \u\|_2^2, \quad \forall \z\in\Z.
	\end{align*}
\end{ass}
It is notable that we do not assume that $r(\w)$ is smooth. 

Our main result in this subsection is presented in the following theorem. 
\begin{thm}[{\bf Result II}]\label{thm:2}  Under {Assumptions~\ref{ass:1},~\ref{ass:2}, and~\ref{ass:3}},  with probability at least $1-\delta$ we have
	\begin{align*}
P(\wh) - P_*\leq O\left( \frac{d \log n + \log(1/\delta)}{n} + \left[\frac{(d\log n + \log(1/\delta)) P_*}{ n}\right]^{\frac{1}{2-\theta}} \right).
	\end{align*}
	When $n \geq  \Omega\left(\left( \alpha^{1/\theta}  d \log n\right)^{2-\theta}\right) $,
	with probability at least $1-\delta$,
	\begin{align*}
	P(\wh) - P_*\leq  O\left(\left[\frac{d\log n + \log(1/\delta)}{n}\right]^{\frac{2}{2-\theta}}+\left[ \frac{(d\log n +\log(1/\delta)) P_*}{ n}\right]^{\frac{1}{2-\theta}} \right).
	\end{align*}
\end{thm}
{\bf Remark:} The constant in big $O$ and $\Omega$ can be seen from the proof, which is tedious and included in the supplement. Here we focus on the  understanding of the results.  First,  the above results are optimistic rates that are no worse than that in  Theorem~\ref{thm:1}. Second, the first result implies that when the optimal risk $P_*$ is less than $O((\frac{d\log n}{n})^{1-\theta})$, the excess risk bound is in the order of $O(\frac{d\log n}{n})$. Third, when the number of samples $n$ is sufficiently large the second result can imply a faster rate than $O(\frac{d\log n}{n})$. Considering smooth functions  presented in Section~\ref{sec:exam} with $\theta=1$, when $n\geq \Omega(\alpha d\log n)$ and $P_*\leq O(d\log n/n)$ (large-sample and small optimal risk), the excess risk can be bounded by $O\left(\left(\frac{d\log n}{n}\right)^2\right)$. In another word, the sample complexity for achieving an $\epsilon$-excess risk bound is given by $\widetilde O\left(\frac{d}{\sqrt{\epsilon}}\right)$.  To the best of our knowledge, the sample complexity of ERM in the order of $1/\sqrt{\epsilon}$ for these examples is the first result appearing in the literature. 

In terms of analysis, we follow the  framework developed in~\citep{DBLP:journals/corr/0005YJ17}, which converts the excess risk bound of $\wh$  into large deviation of gradients. In particular, if we let $F(\w) = \E[f(\w; \z)]$ and $F_n(\w) = \frac{1}{n}\sum_{i=1}^nf(\w; \z_i)$, we prove the following lemma. 
\begin{lem}If we let $\wh^*$ be an optimal solution to $\min_{\w\in\W}P(\w)$ that is closest to $\wh$, then we have
	\begin{equation*}
	\begin{aligned}
	& P(\wh) - P(\wh^*)\\
	&\leq (\left\| \nabla F(\wh) - \nabla F(\wh^*) - [\nabla F_n(\wh) - \nabla F_n(\wh^*)]\right\|_2 + \left\|\nabla F(\wh^*) - \nabla F_n(\wh^*)\right\|_2)\cdot\|\wh - \wh^*\|_2
	\end{aligned}
	\end{equation*}
\end{lem}
Then we use concentration inequalities, covering numbers,  and a refined analysis leveraging the EBC to bound the excess risk, where the refined analysis leveraging the EBC is our main contribution for proving  Theorem~\ref{thm:2}.


\section{Efficient Stochastic Approximation for Lipschitz continuous random functions}
In this section, we will present intermediate rates of an efficient stochastic approximation  algorithm for solving~(\ref{eqn:opt})  adaptive to the EBC under the {Assumption~\ref{ass:1} and~\ref{ass:2}}. Note that (\ref{eqn:opt2}) can be considered as a special case by absorbing $g(\w)$ into $f(\w, \z)$.   


Denote by $\z_1,\ldots \z_k, \ldots$  i.i.d samples drawn sequentially from the distribution $\P$,  by $g_k\in\partial f(\w,\z_k)|_{\w=\w_k}$ a \textit{stochastic subgradient} evaluated at $\w_k$ with sample $\z_k$, and by $\B(\w, R)$  a bounded ball centered at $\w$ with a radius $R$. By the Lipschitz continuity of $f$, we have $\|\partial f(\w, \z)\|_2\leq G$ for $\forall \w\in\mathcal{W},\forall\z\in\Z$.

The proposed adaptive stochastic approximation algorithm is presented in Algorithm~\ref{alg:FSA}, which is referred to as ASA. The updates are divided into $m$ stages, where at each stage a stochastic subgradient method (Algorithm~\ref{alg:sa}) is employed for running $n_0 = \lfloor n/m\rfloor$ iterations with a constant step size $\gamma_k$. The step size $\gamma_k$ will be decreased by half after each stage and the next stage will be warm-started using the solution returned from the last stage as the initial solution. The projection onto the intersection of $\W$ and a shrinking bounded ball at each stage is a commonly used trick for the high probability analysis~\citep{hazan-20110-beyond,juditsky2014,DBLP:journals/corr/abs-1607-01027}.  We  emphasize that the subroutine in ASA can be  replaced by other SA algorithms, e.g., the proximal variant of stochastic subgradient for handling a non-smooth deterministic component such as $\ell_1$ norm regularization~\citep{duchi-2009-efficient},  stochastic mirror descent with with a $p$-norm divergence function~\citep{conf/colt/DuchiSST10}, and etc. Please see an example in the supplement. 

It is worth mentioning that the dividing schema of ASA is due to~\citep{juditsky2014}, which however restricts its analysis to  uniformly convex functions where uniform convexity is a stronger condition than the EBC. ASA is also similar to a recently proposed accelerated stochastic subgradient (ASSG) method under the EBC~\citep{DBLP:journals/corr/abs-1607-01027}. However, the key differences are that (i) ASA is developed for a fixed number of iterations while ASSG is developed for a fixed accuracy level $\epsilon$; (ii) the adaptive iteration complexity of ASSG requires knowing the value of $\theta\in(0,2]$  while ASA does not require the value of $\theta$. As a trade-off, we restrict our attention to $\theta\in(0,1]$. 
		\begin{algorithm}[t]
			\caption{SSG$(\w_1,\gamma,T,\W)$}
			\label{alg:sa}
			\begin{algorithmic}[1]
				\REQUIRE~~$\w_1\in \mathcal{W}$, $\gamma>0$ and  $T$
				\ENSURE~~$\wh_T$
				\FOR {$t=1,\ldots, T$}
				\STATE $\w_{t+1}=\Pi_{\mathcal{W}}(\w_t-\gamma g_t)$
				\ENDFOR
				\STATE $\wh_T=\frac{1}{T+1}\sum_{t=1}^{T+1}{\w_t}$
				\RETURN $\wh_T$
			\end{algorithmic}
		\end{algorithm}
	%
		\begin{algorithm}[t]
			\caption{ASA($\w_1,n,R_0$)}
			\label{alg:FSA}
			\begin{algorithmic}[1]
				\STATE Set $\wh_0=\w_1$, $m=\lfloor \frac{1}{2}\log_2\frac{2n}{\log_2 n}\rfloor-1$, $n_0=\lfloor n/m\rfloor$ 
				\FOR{$k=1,\ldots,m$}
				\STATE Set $\gamma_k=\frac{R_{k-1}}{G\sqrt{n_0+1}}$ and $R_k = R_{k-1}/2$
				\STATE $\wh_{k}=\text{SSG}(\wh_{k-1},\gamma_k,n_0,\W\cap\mathcal{B}(\wh_{k-1},R_{k-1}))$
				\ENDFOR
				\RETURN $\wh_m$
			\end{algorithmic}
		\end{algorithm}

\begin{thm}[{\bf Result III}]
	\label{thm:high}
	Suppose {Assumptions~\ref{ass:1} and~\ref{ass:2}} hold, and $\|\w_1 - \w^*\|_2\leq R_0$, where $\w^*$ is the closest optimal solution to $\w_1$.  
	For  $n\geq 100$ and any $\delta\in(0,1)$, with probability at least $1-\delta$, we have
	\begin{equation*}
		P(\wh_m)-P_*
		\leq O\bigg(\frac{\bar\alpha(\log (n)\log(\log(n)/\delta))}{n}\bigg)^{\frac{1}{2-\theta}}.
	\end{equation*}
	where $\bar\alpha=\max(\alpha G^2, (R_0G)^{2-\theta})$. 
\end{thm}
{\bf Remark:} The significance of the result is that although Algorithm~\ref{alg:FSA} does not utilize any knowledge about EBC,  it is automatically adaptive to the EBC. 
As a final note, the projection onto the intersection of $\W$ and a bounded ball can be efficiently computed by employing the projection onto $\W$ and a binary search for the Lagrangian multiplier of the ball constraint. Moreover, we can replace the subroutine with a slightly different variant of SSG to get around of the  projection onto the intersection of $\W$ and a bounded ball, which is presented in the supplement.

\section{Applications}\label{sec:exam}
In this section, we will present some applications of the developed theories and algorithms in machine learning and other fields by leveraging existing results of EBC. From the last two sections, we can see that $\theta=1$ is a favorable case, which yields the fastest rate in our results. It is obvious that if $f(\w,\z)$ is strongly convex or  $P(\w)$ is strongly convex, then EBC$(\theta=1, \alpha)$ holds. Below we show some examples of problem~(\ref{eqn:opt}) and~(\ref{eqn:opt2}) with $\theta=1$ without strong convexity, which not only recover some known results of fast rate $\widetilde O(d/n)$, but also induce new results of fast rates that are even faster than $\widetilde O(d/n)$. 

\paragraph{Quadratic Problems (QP):}
\begin{equation}\label{eqn:general}
	\min_{\w\in\W}P(\w)\triangleq  \w^{\top}\E_{\z}[A(\z)]\w  +  \w^{\top}\E_{\z'}[\b(\z')] + c
\end{equation}
where  $c$ is a constant. The random function can be taken as $f(\w, \z, \z')=\w^{\top}A(\z)\w + \w^{\top}\b(\z')+c$. We have the following corollary. 

\begin{cor}\label{cor:1}
If  $\E_{\z}[A(\z)]$ is a positive semi-definite  matrix (not necessarily  positive definite) and $\W$ is a bounded polyhedron, then the problem~(\ref{eqn:general})  satisfies EBC$(\theta=1, \alpha)$. Assume that $\max(\|A(\z)\|_2, \|b(\z')\|_2)\leq \sigma<\infty$, then ERM has a fast rate at least $\widetilde O(d/n)$. If $f(\w, \z, \z')$ is further  non-negative, convex and  smooth, then ERM has a fast rate of $\widetilde O((\frac{d}{n})^2 + \frac{dP_*}{n})$ when $n\geq \Omega(d\log n)$. ASA has a convergence rate of $\widetilde O(1/n)$. 
\end{cor}
Next, we present some instances of the quadratic problem~(\ref{eqn:general}).\\
\textit{Instance 1 of QP: minimizing the expected square loss.} Consider the following problem: 
\begin{align}\label{eqn:lps}
	\min_{\w\in\W}P(\w)\triangleq  \E_{\x, y}[(\w^{\top}\x - y)^2]
\end{align}
where $\x\in\X, y\in\mathcal Y$ and $\W$ is a bounded polyhedron (e.g., $\ell_1$-ball or $\ell_\infty$-ball). It is not difficult to show that it is an instance of~(\ref{eqn:general}) and has the property that $f(\w, \z, \z')$ is non-negative, smooth, convex, Lipchitz continuous over $\W$. The convergence results in Corollary~\ref{cor:1} for this instance not only recover some known results of $\widetilde O(d/n)$ rate~\citep{705577,arXiv:1605.01288}, but also imply a faster rate than $\widetilde O(d/n)$ in a large-sample regime and an optimistic case when $n\geq \Omega((P_*\vee 1) d\log n)$, where the latter result is the first such result of its own. 

\textit{Instance 2 of QP.} Let us consider the following problem: 
\begin{align}\label{eqn:eig}
	\min_{\w\in\W} P(\w)\triangleq\E_{\z}[\w^{\top}(S - \z\z^{\top})\w] - \w^{\top}\b
\end{align}
where  $S- \E_{\z}[\z\z^{\top}]\succeq 0$. It is notable that the individual loss functions $f(\w, \z) = \w^{\top}(S - \z\z^{\top})\w - \w^{\top}\b$ might be non-convex. A similar problem as~(\ref{eqn:eig}) could arise in computing the leading  eigen-vector of  $\E[\z\z^{\top}]$ by performing  shifted-and-inverted power method over random samples $\z\sim\P$~\citep{DBLP:conf/icml/GarberHJKMNS16}. 

\paragraph{Piecewise Linear Problems (PLP):}
\begin{align}\label{eqn:plp}
	\min_{\w\in\W}P(\w)\triangleq    \E[f(\w, \z)]
\end{align}
where  $\E[f(\w, \z)]$ is a piecewise linear function and $\W$ is a bounded polyhedron. We have the following corollary. 

\begin{cor}\label{cor:2}
If  $\E[f(\w, \z)]$ is piecewise linear and $\W$ is a bounded polyhedron, then the problem~(\ref{eqn:plp})  satisfies EBC$(\theta=1, \alpha)$. If $f(\w, \z)$ is Lipschitz continuous, then ERM has a fast rate at least $\widetilde O(d/n)$, and  ASA has a convergence rate of $\widetilde O(1/n)$.  If $f(\w, \z)$ is further  non-negative and linear, then ERM has a fast rate of $\widetilde O((\frac{d}{n})^2 + \frac{dP_*}{n})$ when $n\geq \Omega(d\log n)$. 
\end{cor}
\textit{Instance 1 of PLP: minimizing the expected hinge loss for bounded data.}
Consider the following problem: 
\begin{align}\label{eqn:hinge}
	\min_{\|\w\|_p\leq B} P(\w)\triangleq \E_{\x, y}[(1 - y\w^{\top}\x)_+]
\end{align}
where $p=1,  \infty$ and $y\in\{1, -1\}$.   Suppose that $\x\in\X$ is bounded and scaled such that  $|\w^{\top}\x|\leq 1$.  \citet{DBLP:conf/nips/KoolenGE16} has considered this instance with $p=2$ and proved that the Bernstein condition (Definition~\ref{def:2}) holds with $\beta=1$ for the problem~(\ref{eqn:hinge}) when $\E[y\x]\neq 0$ and $|\w^{\top}\x|\leq 1$.  In contrast, we can show that the problem~(\ref{eqn:hinge}) with any $p=1, 2, \infty$ norm constraint~\footnote{The case of $p=2$ is showed later. },  the EBC$(\theta=1, \alpha)$ holds since the objective  $P(\w)=1 - \w^{\top}\E[y\x]$ is essentially a linear function of $\w$. Then all results in Corollary~\ref{cor:2} hold. To the best of our knowledge, the fast rates of ERM and SA for this instance with $\ell_1$ and $\ell_\infty$ norm constraint are the new results. In comparison,   \citet{DBLP:conf/nips/KoolenGE16}'s fast rate of $\widetilde O(1/n)$ only applies to SA and $\ell_2$ norm constraint, and their SA algorithm is not as efficient as our SA algorithm.  

\textit{Instance 2 of PLP: multi-dimensional newsvendor problem.}
Consider a firm that manufactures $p$ products from $q$ resources.  Suppose that a manager must decide on a resource vector $\x\in\R^q_+$ before the product demand vector $\z\in\R^p$ is observed. After the
demand becomes known, the manager chooses a production vector $\y \in\R^p$
so as to maximize the operating profit. Assuming  that  the demand $\z$ is a random vector with discrete probability distribution, the problem is equivalent to 
\begin{align*}
\min_{\x\in\R^q_+, \x\leq \b}\mathbf c^{\top}\x - \E[\Pi(\x; \z)]
\end{align*}
where both $\Pi(\x; \z)$ and $\E[\Pi(\x; \z)]$ are piecewise linear functions~\citep{Kim2015}. Then the problem fits to the setting in Corollary~\ref{cor:2}.

\paragraph{Risk Minimization Problems over an $\ell_2$ ball.} Consider the following problem 
\begin{align}\label{eqn:l2}
	\min_{\|\w\|_2\leq B}P(\w)\triangleq\E_{\z}[f(\w, \z)]
\end{align}
Assuming that $P(\w)$ is convex and $\min_{\w\in\R^d}P(\w)< \min_{\|\w\|_2\leq B}P(\w)$, 
we can show that EBC$(\theta=1, \alpha)$ holds  (see supplement). Using this result, we can easily show that the considered problem~(\ref{eqn:hinge}) with $p=2$ satisfies EBC$(\theta=1, \alpha)$. As another corollary, we have the following result.  
\begin{cor}\label{cor:3}
If  $f(\w,\z) = (\w^{\top}\x - y)^2$ is the square loss and $\x, y$ are bounded, then there exists $\theta\in(0,1]$ such that the problem~(\ref{eqn:l2}) with square loss satisfies EBC$(\theta, \alpha)$. 
As a result, the proposed ASA has a convergence rate ranging from  $\widetilde O(1/n^{1/(2-\theta)})$ to $\widetilde O(1/n)$ depending on the data.
\end{cor}
{\bf Remark:} In this  corollary, we focus on the result for SA, since fast rate of ERM for minimizing expected square loss has been established in literature~(e.g., \citep{705577,arXiv:1605.01288}) by using other techniques and conditions. Efficient SA for minimizing expected square loss under an $\ell_2$-norm constraint with a convergence rate faster than $O(1/\sqrt{n})$ remains rare. For comparison, we compare with two works~\citep{DBLP:conf/nips/BachM13,DBLP:journals/corr/MahdaviJ14}.  \citet{DBLP:journals/corr/MahdaviJ14} proposed a SA algorithm  based on online Newton method for exp-concave loss, which could enjoy  a fast rate of  $\widetilde O(d/n)$ under certain conditions of the data. However, their  algorithm is not as efficient as the proposed ASA due to the online Newton step.  \citet{DBLP:conf/nips/BachM13} analyzed averaged stochastic gradient descent for minimizing expected square loss {\it without any constraint} and established a fast rate of $O(d/n)$ in expectation. However, their convergence result is not a high probability result. 

\paragraph{ $\ell_1$ Regularized Risk Minimization Problems.} For $\ell_1$ regularized risk minimization: 
\begin{align}\label{eqn:l1r}
	\min_{\|\w\|_1\leq B}  P(\w)\triangleq \E[f(\w; \z)]+ \lambda \|\w\|_1,
\end{align} 
we have the following corollary. 
\begin{cor}\label{cor:4}
If the first component is quadratic as in~(\ref{eqn:general}) or is piecewise linear, then the problem~(\ref{eqn:l1r})  satisfies EBC$(\theta=1, \alpha)$. If the random function is Lipschitz continuous, then ERM has a fast rate at least $\widetilde O(d/n)$, and  ASA has a convergence rate of $\widetilde O(1/n)$.  If $f(\w, \z)$ is further  non-negative, convex and smooth, then ERM has a fast rate of $\widetilde O((\frac{d}{n})^2 + \frac{dP_*}{n})$ when $n\geq \Omega(d\log n)$. 
\end{cor}
To the best of our knowledge, this above general result is the first of its kind. 

Next, we show some instances satisfying EBC$(\theta, \alpha)$ with $\theta<1$.  Consider the problem below: 
\begin{align}
	\min_{\w\in\W}F(\w)\triangleq  P(\w)+ \lambda\|\w\|_p^p
\end{align}
where $P(\w)$  is quadratic as in~(\ref{eqn:general}), and  $\W$ is a bounded polyhedron.   
In the supplement, we prove that EBC$(\theta=2/p, \alpha)$ holds.  


\begin{figure}[ttt!]
	\begin{minipage}[t]{0.5\textwidth}
		\centering
		\includegraphics[scale=0.3]{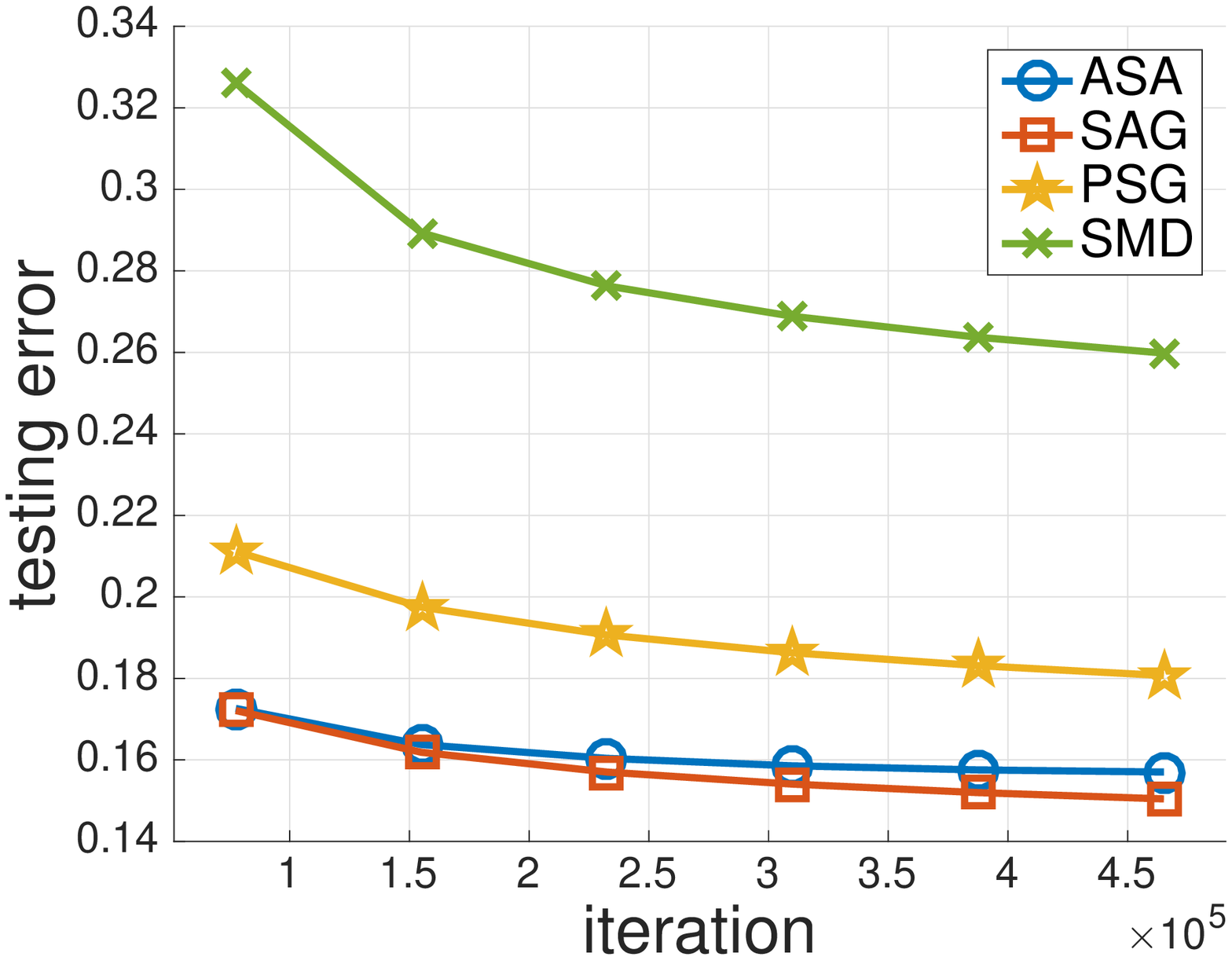}
		\caption*{(a) rcv1\_binary}
		\label{fig:e}
	\end{minipage}
	\begin{minipage}[t]{0.5\textwidth}
		\centering
		\includegraphics[scale=0.3]{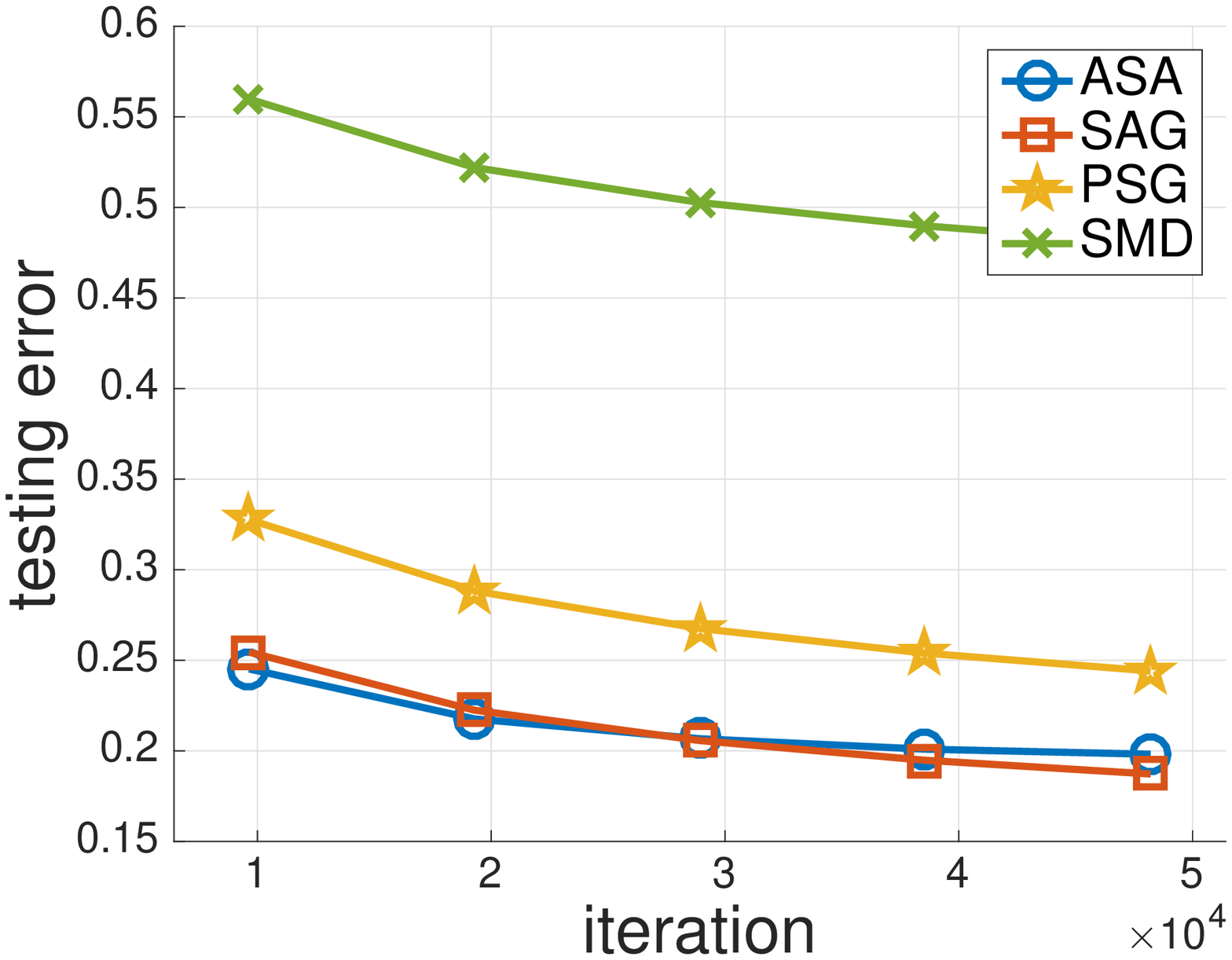}
		\caption*{(b) real-sim}
		\label{fig:f}
	\end{minipage}
	
	\vspace*{0.2in}
	\begin{minipage}[t]{0.5\textwidth}
		\centering
		\includegraphics[scale=0.3]{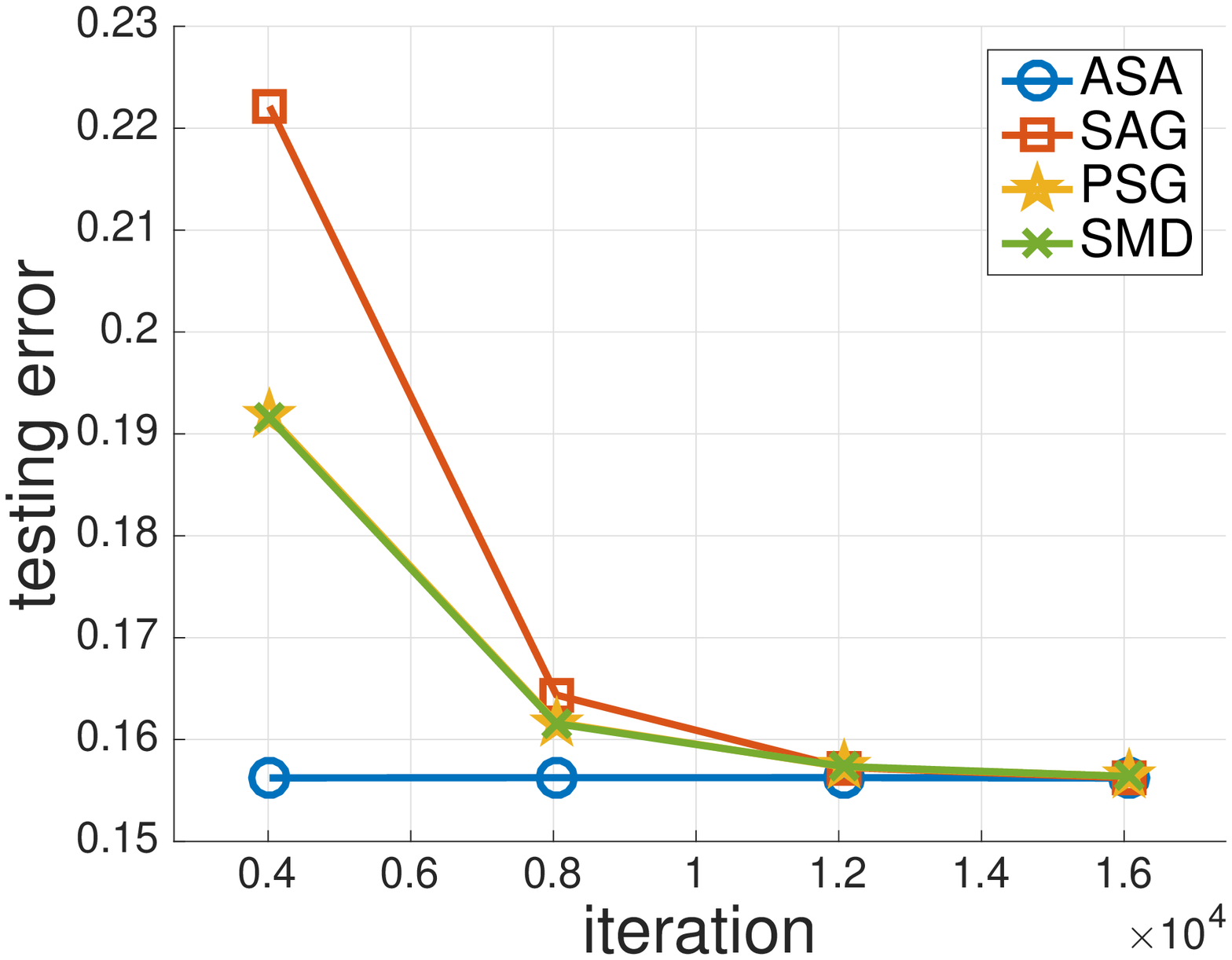}
		\caption*{(c) E2006-tfidf}
		\label{fig:a}
	\end{minipage}
	\begin{minipage}[t]{0.5\textwidth}
		\centering
		\includegraphics[scale=0.3]{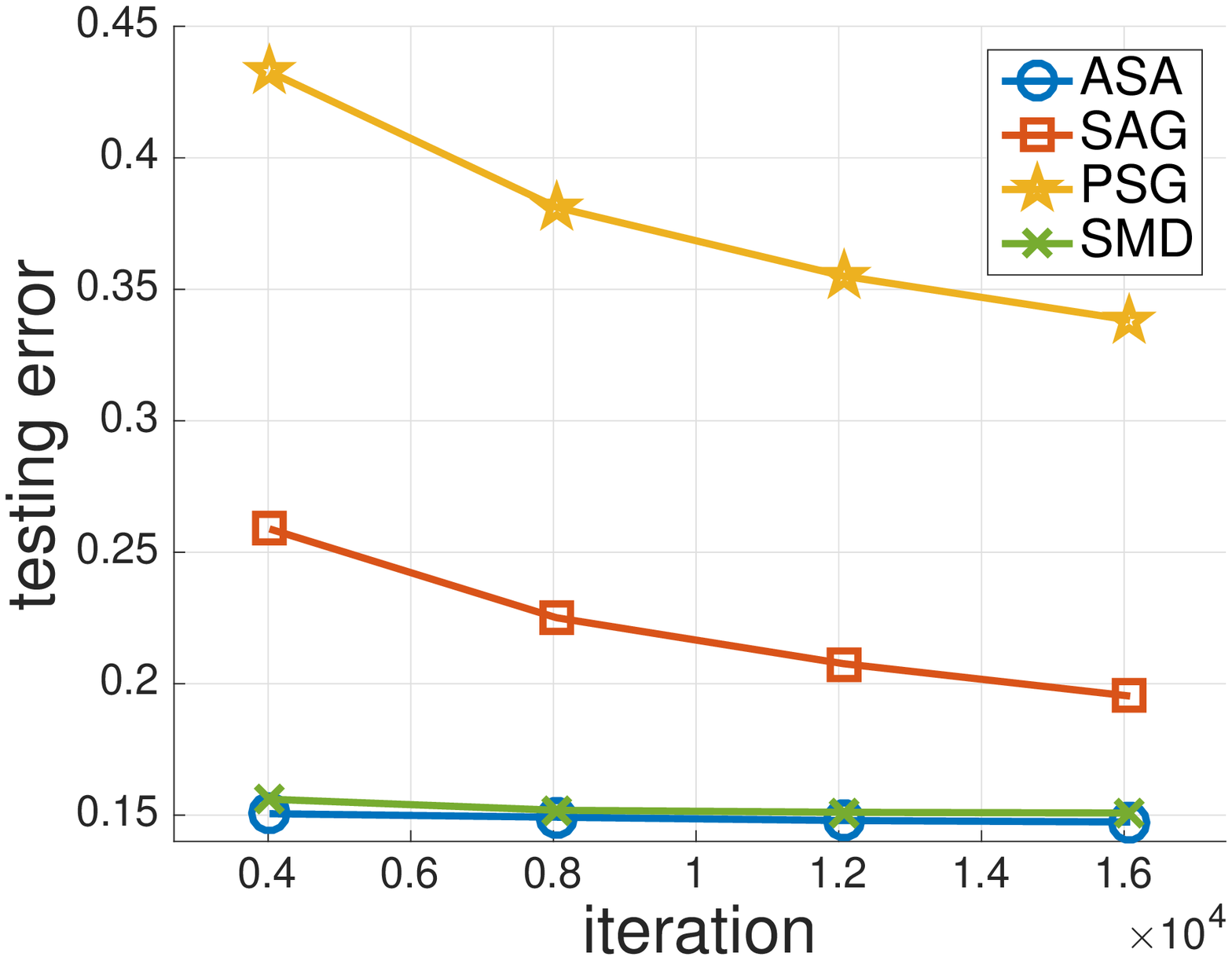}
		\caption*{(d) E2006-log1p}
		\label{fig:a1}
	\end{minipage}
	\caption{Testing Error vs Iteration of ASA and other baselines for SA}
	\label{fig:p2}
\end{figure}

\section{A Case Study for ASA}
In this section, we provide some empirical evidence to support the effectiveness of the proposed ASA algorithm. In particular, we will consider solving an $\ell_1$ regularized expected square loss minimization problem~(\ref{eqn:l1r}) for learning a predictive model. We compare with two baselines whose convergence rate are known as $O(1/\sqrt{n})$, namely proximal stochastic gradient (PSG) method~\citep{duchi-2009-efficient}, and stochastic mirror descent (SMD) method using a $p$-norm divergence function ($p=2\log d$) other than the Euclidean function. For SMD, we implement the algorithm proposed in~\citep{DBLP:journals/jmlr/Shalev-ShwartzT11}, which was proposed for solving~(\ref{eqn:l1r}) and could be effective for very high-dimensional data. For ASA, we implement two versions that use PSG and SMD as the subroutine and report the one that gives the best performance. The two versions differ in using the Euclidean norm or the $p$-norm for measuring distance. 
Since the comparison is focused on the testing error, we also include another strong baseline, i.e, stochastic average gradient (SAG) with a constant step size, which simply minimizes the expected square loss without any constraints or regularizations~\citep{DBLP:conf/nips/BachM13}.


We use four benchmark datasets from libsvm website\footnote{http://www.csie.ntu.edu.tw/~cjlin/libsvmtools/datasets/}, namely, real-sim, rcv1\_binary,  E2006-tfidf, E2006-log1p, whose dimensionality is  20958, 47236, 150360, 4272227, respectively.    
We  divide each dataset into three sets, respectively  training, validation, and testing. For E2006-tfidf and E2006-log1p dataset, we randomly split the given testing set into half validation and half testing. For the dataset real-sim  which do not explicitly provides a testing set, we randomly split the entire data into 4:1:1 for training, validation, and testing. For rcv1\_binary, despite that the test set is given, the size of the training set is relatively small. Thus we first combine the training and the testing sets and then follow the above procedure to split it.

The involved  parameters of each algorithm are tuned based on the validation data.  
With the selected parameters, we run each algorithm by passing through  training examples once and evaluate intermediate models on the testing data to compute the testing error measured by square loss.  
The results on different data sets averaged over 5 random runs over shuffled training examples are shown in  Figure \ref{fig:p2}. From the testing curves, we can see that the proposed ASA has similar convergence rate to SAG on  two relatively low-dimensional  data sets. This is not surprise since both algorithms enjoy an $\widetilde O(1/n)$ convergence rate indicated by their theories. For the data set E2006-tfidf and E2006-log1p, we observe that ASA converges much faster than SAG, which is due to the presence  of $\ell_1$ regularization.   In addition, ASA converges much faster than SGD and SMD with one exception on E2006-log1p, on which ASA performs slightly better than SMD. 



\section{Conclusion}
We have comprehensively studied statistical learning under the error bound condition for both empirical risk minimization and stochastic approximation. We established the connection between the error bound condition and previous conditions for developing fast rates of empirical risk minimization for Lipschitz continuous loss functions. We also developed improved rates for non-negative and smooth convex loss functions, which induce faster rates that were not achieved before. Finally, we analyzed an efficient ``parameter"-free stochastic approximation algorithm under the error bound condition and showed that it is automatically adaptive to the  error bound condition. Applications in machine learning and other fields are considered and empirical studies corroborate the fast rate of the developed algorithms.  

{
	\bibliography{ref,all}
	\bibliographystyle{plain}}

\newpage
\appendix

\section{Proof of Lemma~\ref{lem:1}}
\begin{proof}
	The proof follows similarly as the proof of Theorem 5.4 in~\citep{DBLP:journals/jmlr/ErvenGMRW15}. 
	Let us fix an arbitrary $\w\in\W$ and its closest optimal solution $\w^*\in\W_*$. Let $X = f(\w, \z) - f(\w^*, \z)$ be a random variable due to $\z$. Then $|X|\leq 2GR\triangleq a$.   Let $b>0$ be any finite constant, $\kappa(x)= (e^x - x - 1)/x^2$ for $x\neq 0$ and $\kappa(0) = 1/2$, $c_1^b = 1/\kappa(2ba)$. Let $B=\alpha G^2$ and $v(x) = \frac{c_1^b}{B}x^{1-\theta}\wedge b$. Let $\varepsilon\geq 0$ and set $\eta = v(\varepsilon)\leq \frac{c^b_1}{B}\varepsilon^{1-\theta}$. 
	
	According to our analysis in the paper, we have established a similar  condition to the Bernstein condition under our conditions, i.e., 
	\begin{align*}
		\E_{\z}[(f(\w, \z) - f(\w^*, \z))^2] \leq  B(\E_\z[f(\w, \z) - f(\w^*, \z)])^\theta
	\end{align*}
	where $B= \alpha G^2$. Then 
	\begin{align*}
		\text{Var}[(f(\w, \z) - f(\w^*, \z))] \leq  B(\E_\z[f(\w, \z) - f(\w^*, \z)])^\theta
	\end{align*}
	
	First, when $\varepsilon=0$ we have $\eta = 0$, the  $ \E[e^{-\eta X}]\leq e^{\eta\varepsilon}$ hold trivially.  Thus we focus on the case $\varepsilon>0$, which implies that $\eta>0$. Then Lemma 5.6 in~\citep{DBLP:journals/jmlr/ErvenGMRW15} applied to the random variable $\eta$ gives 
	\begin{align*}
	\E[X] + \frac{1}{\eta}\log \E[e^{-\eta X}]\leq \kappa(2ba)\eta \text{Var}(X)
	\leq \kappa(2ba)\eta B(\E[X])^\theta
	\leq \varepsilon^{1-\theta}(\E[X])^{\theta}.
	\end{align*}
	If $\varepsilon\leq \E[X]$, then $\varepsilon^{1-\theta}(\E[X])^{\theta}\leq \E[X]$, which implies $\frac{1}{\eta}\log \E[e^{-\eta X}]\leq 0\leq \varepsilon$.  This establishes the second part and the first part for $\varepsilon\leq \E[X]$. For $\varepsilon\geq \E[X]$, we have $\varepsilon^{1-\theta}(\E[X])^{\theta}\leq \varepsilon$. Then due to $\E[X]\geq 0$, we have $\frac{1}{\eta}\log \E[e^{-\eta X}]\leq \varepsilon$.
\end{proof}
\section{Proof of Theorem~\ref{thm:1}}
\begin{proof}
	Let $F_\w(\z) = f(\w, \z) - f(\w^*, \z)$, where $\w^*$ is the closest optimal solution to $\w$. Denote by $B= 2GR$. It is clear that $F_\w(\z)\leq B$. The goal is to show that with high probability, ERM does not select any $\w\in\W$ whose excess risk $P(\w) - P_* = \E_\z[F_{\w}(\z)]$ is large than $\left(\frac{a}{n}\right)^{\frac{1}{2-\theta}}$ for some constant $a$. Clearly, with probability $1$ ERM will never select any $\w$ for which both $F_\w(\z)>0$ almost surely and with some positive probability $F_\w(\z)>0$. These predictors are called the empirically inadmissible models. For any $\gamma_n>0$, let $\W_{\geq \gamma_n}$ denote the subclass of models by starting with $\W$, retaining only models whose excess risk is at least $\gamma_n$, and further removing the empirically inadmissible models. 
	
	The goal now can be expressed equivalently as showing that, with high probability, ERM does not select any model $\w\in\W_{\geq\gamma_n}$, where $\gamma_n = \left(\frac{a}{n}\right)^{\frac{1}{2-\theta}}$. Let $\mathcal W_{\geq\gamma_n, \varepsilon}$ be the  optimal proper $(\varepsilon/(2G))$-cover of $\W_{\geq\gamma_n}$.  Note that this cover induces an $\varepsilon$-cover in sup norm over the function class $\{F_\w: \w\in\W_{\geq\gamma_n}\}$.  To see this, for any $\w\in\W_{\geq \gamma_n}$, there exists $\wt\in\W_{\geq\gamma_n, \varepsilon}$ such that $\|\w - \wt\|_2\leq \varepsilon/(2G)$. As a result, 
	\begin{align*}
		\sup_\z|F_{\w}(\z) - F_{\wt}(\z)|
		&= \sup_{\z}|f(\w, \z) - f(\wt, \z)|  + \sup_{\z}|f(\w^*, \z) - f(\wt^*, \z)|\\
		&\leq G\|\w - \wt\|_2 +G\|\w^* - \wt^*\|_2\leq 2G\|\w - \wt\|_2\leq \varepsilon,
	\end{align*}
	where $\w^*, \wt^*$ are projections of $\w$ and $\wt$ onto $\W_*$ and the last inequality uses the non-expansiveness of the projection onto $\W_*$, which is convex due to the convexity of $P(\w)$ and $\W$. Observe that the $\epsilon$-cover  of $\W_{\geq \gamma_n}\subseteq\mathcal B^d(R)$ has cardinality at most $\left(\frac{4R}{\varepsilon}\right)^d$, and the cardinality of an optimal proper $\varepsilon$-cover is at most the cardinality of an optimal $\varepsilon/2$-cover.  It hence follows that $|\W_{\geq\gamma_n, \varepsilon}|\leq \left(\frac{16GR}{\varepsilon}\right)^d$. 
	
	Let us consider a fixed $\w\in\W_{\geq \gamma_n, \varepsilon}$ and its closest optimal solution $\w^*\in\W_*$. According to Lemma~\ref{lem:1}, we have 
	\[
	\E_\z[e^{-v(\gamma_n) F_\w(\z)}] \leq 1
	\]
	Then using Theorem 13 in~\citep{DBLP:journals/corr/GrunwaldM16}, where we set $u = B$ and $c=1$, for all $\eta\in(0, v(\gamma_n))$ we have
	\begin{align*}
		\gamma_n\leq \E_{\z}[F_\w(\z)]   \leq -\frac{\eta B + 1}{1 - \eta/v(\gamma_n)}\frac{1}{\eta}\log\E_{\z}[e^{-\eta F_\w(\z)}]
	\end{align*}
	Let $\eta = v(\gamma_n)/2$, we have 
	\[
	\log \E_{\z}[e^{-(v(\gamma_n) /2)F_\w(\z) }]\leq -\frac{0.5v(\gamma_n)}{Bv(\gamma_n) +  2}\gamma_n
	\]
	Applying Theorem 1 in~\citep{DBLP:conf/nips/MehtaW14} with $t = \frac{\gamma_n}{2}$, we have 
	\begin{align*}
	\Pr\left(\frac{1}{n}\sum_{i=1}^nF_\w(\z_i)\leq \frac{\gamma_n}{2}\right) \leq\exp\left(-\frac{0.5v(\gamma_n)}{Bv(\gamma_n) +  2}n\gamma_n +\frac{v(\gamma_n)\gamma_n}{4}\right).
	\end{align*}
	Assume that $\left(\frac{a}{n}\right)^{\frac{1-\theta}{2-\theta}}\leq \alpha bG^2\kappa(4GRb)$, i.e., $n\geq a\left(\alpha b G^2\kappa(4GRb)\right)^{(2-\theta)/(1-\theta)}$, which implies that $v(\gamma_n) = c\left(\frac{a}{n}\right)^{\frac{1-\theta}{2-\theta}}\wedge b = c\left(\frac{a}{n}\right)^{\frac{1-\theta}{2-\theta}}$ by noting the value of $c = 1/(\alpha G^2\kappa(4GRb))$ in Lemma~\ref{lem:1}.  Further we assume $n\geq a(0.5Bc)^{\frac{2-\theta}{1-\theta}}$. Hence $Bv(\gamma_n)\leq 2$. 
	
	\begin{align*}
	&\Pr\left(\frac{1}{n}\sum_{i=1}^nF_\w(\z_i)\leq \frac{\gamma_n}{2}\right)
	\leq \exp\left(-\frac{0.5v(\gamma_n)}{Bv(\gamma_n) +  2}n\gamma_n +\frac{v(\gamma_n)\gamma_n}{4}\right)\\
	&\leq  \exp\left(-0.125v(\gamma_n)n\gamma_n +\frac{v(\gamma_n)\gamma_n}{4}\right) 
	= \exp\left(-0.125c a +\frac{ca}{4n}\right) \\
	&\leq \exp\left(-0.375c a \right),
	\end{align*}
	where we use $n\geq 1$. 
	
	As a result, we have 
	\[
	\Pr\left(\frac{1}{n}\sum_{i=1}^nF_\w(\z_i)\leq \frac{\gamma_n}{2}\right)\leq  \exp\left(-0.375ca \right)
	\]
	Taking a union bound over $\W_{\geq\gamma_n, \varepsilon}$ 
	we have that 
	\begin{align*}
	&\Pr\left(\exists \w\in\W_{\geq \gamma_n, \varepsilon}, \frac{1}{n}\sum_{i=1}^nF_\w(\z_i)\leq \frac{\gamma_n}{2}\right)\leq \left(\frac{16GR}{\varepsilon}\right)^d\exp\left(-0.375ca\right)
	\end{align*}
	
	Taking $\varepsilon = \frac{1}{2n^{1/(2-\theta)}}$  and $a  = \frac{3}{c} (d\log(32GRn^{1/(2-\theta)}) + \log(1/\delta)) $,  with probability  $1-\delta$ for all $\w\in\W_{\geq\gamma_n, \varepsilon}$, we have $\frac{1}{n}\sum_{i=1}^nF_\w(\z_i)\geq \frac{a^{1/(2-\theta)}}{2n^{1/(2-\theta)}}$. 
	
	Now, since $\sup_{\w\in\W_{\geq\gamma_n}}\min_{\wt\in\W_{\geq\gamma_n, \varepsilon}}\|F_\w - F_\w\|_\infty\leq \varepsilon = \frac{1}{2n^{1/(2-\theta)}}$, and by increasing $a$ by 1 to guarantee that $a>1$, with probability $1-\delta$, for all $\w\in\W_{\geq \gamma_n}$,  $\frac{1}{n}\sum_{i=1}^nF_\w(\z_i)>0$.
\end{proof}

\section{Proof of Theorem~\ref{thm:2}}
\begin{proof}
	We first prove the following theorem. Theorem~\ref{thm:2} is a corollary of the following theorem by setting $\varepsilon = 1/n$.  To be more general, we consider the stochastic composite optimization, 
	\begin{align}\label{eqn:sco2}
		\min_{\w\in\W}P(\w) \triangleq  \E_{\z}[f(\w, \z)] + r(\w)
	\end{align}
	We abuse the notation $F(\w) = \E_{\z}[f(\w,\z)]$ and $F_n(\w) = \frac{1}{n}\sum_{i=1}^nf(\w, \z_i)$ in the following proof. In the following Theorem, we assume the problem~(\ref{eqn:sco2}) satisfies the EBC$(\theta, \alpha)$. 
	\begin{thm} \label{thm:smooth:convex}
		Let $\varepsilon>0$ be any constant and $C(\varepsilon)=2(\log(2/\delta) + d\log(6R/\varepsilon)$. Under {\bf Assumptions~\ref{ass:1},~\ref{ass:2},~\ref{ass:3}}, and that $r(\w)$ is convex and $G'$-Lipschitz continuous over $\W$, with probability at least $1 - 2\delta$, we have
		\begin{equation*}
			\begin{split}
				P(\wh) - P_* &\leq  \frac{4(6LR^2+\bar GR)C(\varepsilon)}{n}
				+  2\left(1\vee \alpha^{1/\theta}\right) \left(\frac{4LC(\epsilon)P_*}{n}\right)^{\frac{1}{2 - \theta}} 
			+ 2\left(12RL + \frac{\bar G}{4} + \frac{4LRC(\varepsilon)}{n}\right)\varepsilon,
			\end{split}
		\end{equation*}
		where $\bar G= G + G'$. 
		Furthermore, if
		$n \geq \left(256LC(\varepsilon)\alpha^{\frac{1}{\theta}}\right)^{(2-\theta)}$,
		we also have
		\begin{align*}
			&P(\wh) - P_*\leq 34LC(\varepsilon)\left(\frac{1}{n}\right)^{\frac{2}{2-\theta}}
			+2 \left(1 \vee 4\alpha^{1/\theta}\right)\left(\frac{\bar GC(\varepsilon)}{n}\right)^{\frac{2}{2-\theta}} + 2\left(1\vee64\alpha^{1/\theta}\right) \left(\frac{4LC(\varepsilon)P_*}{n}\right)^{\frac{1}{2 - \theta}} \\
			&+ 4LC(\varepsilon)\left(1 \vee 64\alpha^{1/\theta}\right)\left(\frac{\varepsilon}{n}\right)^{\frac{2}{2-\theta}} +  12L\left(1 \vee 64\alpha^{1/\theta}\right) \varepsilon^{\frac{2}{2-\theta}} +  2\left(1\vee 64\alpha^{1/\theta}\right) \left(\frac{4L\bar GC(\varepsilon)\varepsilon}{n}\right)^{\frac{1}{2 - \theta}}.\notag
		\end{align*}
	\end{thm}
	
	To prove the theorem, we need the following two lemmas. 
	\begin{lem} \label{lem:net} Under {\bf Assumptions~\ref{ass:1}}, with probability at least $1-\delta$, for any $\w \in \W$, we have
		\begin{align*}
			&\left\| \nabla F(\w) - \nabla F(\w^*) -  [\nabla F_n(\w) - \nabla F_n(\w^*)]\right\|_2 \\
			&\leq \frac{LC(\varepsilon) \|\w - \w^*\|_2 }{n}  + \frac{2LC(\varepsilon)\varepsilon }{n}
			+ \sqrt{\frac{LC(\varepsilon) (P(\w)-  P_* )}{n} }     + 2\sqrt{\frac{ L\bar GC(\varepsilon) \varepsilon }{n}} 
			+ 4L\varepsilon.
		\end{align*}
		where $\w^*$ is the closest optimal solution to $\w$ and $C(\varepsilon)$ is define in Theorem~\ref{thm:smooth:convex}. 
	\end{lem}
	
	\begin{lem} \label{lem:vec:con} Under {\bf Assumption~\ref{ass:1}}, with probability at least $1-\delta$, for any $\w_*\in\W_*$, we have
		\begin{equation} \label{eqn:add:2}
			\left\|\nabla F(\w_*) - \nabla F_n(\w_*)\right\|_2 \leq \frac{GC(\varepsilon)}{n} + \sqrt{\frac{4 LC(\varepsilon) P_* }{n}} + 2L\varepsilon.
		\end{equation}
	\end{lem}
	
	\begin{lem}\label{lem:useEBC}
		Let $A$ be a nonnegative number. Under the EBC$(\theta, \alpha)$ condition with $\theta\in(0,1]$ and $0<\alpha<\infty$, for any $\epsilon>0$ and $\w\in\W$, we have
		\begin{align*}
			\left\| \w - \w^* \right\|_2 \sqrt{A}\leq \left(1\vee \frac{\alpha^{1/\theta}}{4\epsilon}\right) A^{\frac{1}{2 - \theta}} + \epsilon(P(\w) - P_*)
		\end{align*}
	\end{lem}

	\subsection{Proof of Theorem~\ref{thm:smooth:convex}}
	\begin{proof}
		\begin{equation*}
			\begin{aligned}
				& P(\wh) - P(\wh^*)  \leq  \langle \partial P(\wh), \wh - \wh^* \rangle\\
				&= \langle \partial P(\wh) - \partial P(\wh^*), \wh - \wh^* \rangle + \langle \partial P(\wh^*), \wh - \wh^*\rangle \\
				&= \langle \partial P(\wh) - \partial P(\wh^*) - [\partial P_n(\wh) - \partial P_n(\wh^*)], \wh - \wh^* \rangle 
				\\
				&+ \langle \partial P_n(\wh) - \partial P_n(\wh^*) + \partial P(\wh^*), \wh - \wh^* \rangle \\
				&=   \langle \partial P(\wh) - \partial P(\wh^*) - [\partial P_n(\wh) - \partial P_n(\wh^*)], \wh - \wh^* \rangle + 
				 \langle \partial P(\wh^*) - \partial P_n(\wh^*), \wh - \wh^*\rangle \\
				 &+ \langle \partial P_n(\wh), \wh - \wh^* \rangle
			\end{aligned}
		\end{equation*}
		According to the optimality condition of $\wh$, there exists $\v\in\partial r(\wh)$  such that $ \langle \nabla F_n(\wh) + \v, \wh - \wh^*\rangle\leq 0$. Let $\partial P_n(\wh) = \nabla F_n(\wh) + \v$ and $\partial P(\wh) = \nabla F(\wh) + \v$ in the above inequality, we have
		\begin{equation*}
			\begin{aligned}
				& P(\wh) - P(\wh^*)\\& \leq   \langle \nabla  F(\wh) - \nabla  F(\wh^*) - [\nabla  F_n(\wh) - \nabla  F_n(\wh^*)], \wh - \wh^* \rangle + \langle \nabla  F(\wh^*) - \nabla  F_n(\wh^*), \wh - \wh^*\rangle \\
				&\leq (\left\| \nabla F(\wh) - \nabla F(\wh^*) - [\nabla F_n(\wh) - \nabla F_n(\wh^*)]\right\|_2 + \left\|\nabla F(\wh^*) - \nabla F_n(\wh^*)\right\|_2)\cdot\|\wh - \wh^*\|_2
			\end{aligned}
		\end{equation*}
		
		Using the Lemma~\ref{lem:net} and Lemma~\ref{lem:vec:con} to proceed bounding the above inequality, with probability at least $1-2\delta$, we have
		\begin{equation} \label{eqn:add:3}
			\begin{split}
				& P(\wh) - P_*  
				\leq  \frac{LC(\varepsilon)\|\wh - \wh^*\|_2^2 }{n} +  \frac{\bar GC(\varepsilon)\left\| \wh - \wh^* \right\|_2 }{n}   +  \frac{2LC(\varepsilon) \varepsilon \|\wh - \wh^*\|_2}{n}  + 6L \varepsilon \left\| \wh - \wh^* \right\|_2 \\
				&+ \left\| \wh - \wh^* \right\|_2 \sqrt{\frac{LC(\varepsilon)(P(\wh)-  P_* )}{n}}   + \left\| \wh - \wh^* \right\|_2 \sqrt{\frac{4 LC(\varepsilon) P_*}{n}}+\left\| \wh - \wh^* \right\|_2 \sqrt{\frac{4L\bar GC(\varepsilon) \varepsilon}{n}}. \\
			\end{split}
		\end{equation}
		Next, we will bound the three terms that have a $1/\sqrt{n}$ factor.  

		\begin{equation}
		\begin{aligned}
			&\left\| \wh - \wh^* \right\|_2 \sqrt{\frac{LC(\varepsilon)(P(\wh)-  P_* )}{n}} \leq \frac{LC(\varepsilon) \left\| \wh - \wh^* \right\|_2^2}{ n} + \frac{P(\wh)-  P_* }{4}, \label{eqn:inequality:1} 
			\end{aligned}
			\end{equation}

			\begin{equation}
			\label{eqn:inequality:2}
			\left\| \wh - \wh^* \right\|_2 \sqrt{\frac{4L\bar GC(\varepsilon)  \varepsilon}{n}} \leq \frac{LC(\varepsilon) \left\| \wh - \wh^* \right\|_2^2}{ n} +  \bar G \varepsilon
			\end{equation}
			\begin{equation}
			\begin{aligned}
			&\left\| \wh - \wh^* \right\|_2 \sqrt{\frac{4 LC(\varepsilon) P_*}{n}}\leq  \left(1\vee \alpha^{1/\theta}\right) \left(\frac{4LC(\varepsilon)P_*}{n}\right)^{\frac{1}{2 - \theta}} + \frac{P(\wh) - P_*}{4}\label{eqn:inequality:3}
			\end{aligned}
			\end{equation}
	, where the last inequality follows Lemma~\ref{lem:useEBC}. 
		Combining the inequalities in~(\ref{eqn:add:3}),~(\ref{eqn:inequality:1}),~(\ref{eqn:inequality:2}), and~(\ref{eqn:inequality:3}), with probability $1-\delta$ we have
		\begin{align*}
			&\frac{P(\wh) - P_*}{2}\\
			&\leq  \frac{3LC(\varepsilon)\|\wh - \wh^*\|_2^2 }{n} +  \frac{\bar GC(\varepsilon)\left\| \wh - \wh^* \right\|_2 }{n}   +  \frac{2LC(\varepsilon) \varepsilon \|\wh - \wh^*\|_2}{n} + 6L \varepsilon \left\| \wh - \wh^* \right\|_2   \\&+  \bar G \varepsilon +  \left(1\vee \alpha^{1/\theta}\right) \left(\frac{4LC(\varepsilon)P_*}{n}\right)^{\frac{1}{2 - \theta}} \\
			&\leq  \frac{(12LR^2+2\bar GR)C(\varepsilon)}{n}   +  \left(1\vee \alpha^{1/\theta}\right) \left(\frac{4LC(\varepsilon)P_*}{n}\right)^{\frac{1}{2 - \theta}} + \left(12RL + \bar G + \frac{4LRC(\varepsilon)}{n}\right)\varepsilon,
		\end{align*}
		which finishes the first part of the theorem. 
		
		To prove the second part, we need more refined analysis. The following inequalities will be proved later. 
\begin{align}
&\frac{LC(\varepsilon)\|\wh - \wh^*\|_2^2 }{n}\leq \max\left(LC(\varepsilon)\left(\frac{1}{n}\right)^{\frac{2}{2-\theta}} ,  \epsilon(P(\wh) - P_*)\right), \: n\geq \left(LC(\varepsilon)\alpha^{\frac{1}{\theta}}/\epsilon\right)^{(2-\theta)}\label{refined:eq1}\\
&\left\| \wh - \wh^* \right\|_2 \sqrt{\frac{LC(\varepsilon)(P(\wh)-  P_* )}{n}} \leq \epsilon(P(\wh) - P_*) + \frac{LC(\varepsilon)\|\wh - \wh^*\|_2^2}{4\epsilon n}\nonumber \\
& \leq \epsilon (P(\wh) - P_*) +\max\left( \frac{LC(\varepsilon)}{\epsilon}\left(\frac{1}{n}\right)^{\frac{2}{2-\theta}},\epsilon (P(\wh) - P_*)\right),\: n\geq \left(LC(\varepsilon)\alpha^{\frac{1}{\theta}}/\epsilon^2\right)^{(2-\theta)}\label{refined:eq2}\\
&\frac{GC(\varepsilon)\left\| \wh - \wh^* \right\|_2 }{n} \leq \left\{\left(1 \vee \frac{\alpha^{1/\theta}}{4\epsilon}\right)\left(\frac{ GC(\varepsilon)}{n}\right)^{\frac{2}{2-\theta}} + \epsilon(P(\wh) - P_*)\right\}\label{refined:eq3}\\
&\frac{2LC(\varepsilon) \varepsilon \|\wh - \wh^*\|_2}{n}\leq2LC(\varepsilon)\left\{\left(1 \vee \frac{\alpha^{1/\theta}}{4\epsilon}\right)\left(\frac{\varepsilon}{n}\right)^{\frac{2}{2-\theta}} + \epsilon(P(\wh) - P_*)\right\} \label{refined:eq4}\\
&6L \varepsilon \left\| \wh - \wh^* \right\|_2\leq 6L\left\{\left(1 \vee \frac{\alpha^{1/\theta}}{4\epsilon}\right) \varepsilon^{\frac{2}{2-\theta}} + \epsilon(P(\wh) - P_*)\right\}\label{refined:eq5}\\
& \left\| \wh - \wh^* \right\|_2 \sqrt{\frac{4 LC(\varepsilon)P_*}{n}}\leq \left(1\vee \frac{\alpha^{1/\theta}}{4\epsilon}\right) \left(\frac{4LC(\varepsilon)P_*}{n}\right)^{\frac{1}{2 - \theta}} + \epsilon(P(\wh) - P_*)\label{refined:eq6}\\
&\left\| \wh - \wh^* \right\|_2 \sqrt{\frac{4LGC(\varepsilon) \varepsilon}{n}}\leq  \left(1\vee \frac{\alpha^{1/\theta}}{4\epsilon}\right) \left(\frac{4LGC(\varepsilon)\varepsilon}{n}\right)^{\frac{1}{2 - \theta}} + \epsilon(P(\wh) - P_*)\label{refined:eq7}
\end{align}
Plugging appropriate values of $\epsilon$ in each inequality, we have 
\begin{align*}
&\frac{P(\wh) - P_*}{2}\leq 17LC(\varepsilon)\left(\frac{1}{n}\right)^{\frac{2}{2-\theta}}
+ \left(1 \vee 4\alpha^{1/\theta}\right)\left(\frac{GC(\varepsilon)}{n}\right)^{\frac{2}{2-\theta}} + 2LC(\varepsilon)\left(1 \vee 64\alpha^{1/\theta}\right)\left(\frac{\varepsilon}{n}\right)^{\frac{2}{2-\theta}}\\
& +  6L\left(1 \vee 64\alpha^{1/\theta}\right) \varepsilon^{\frac{2}{2-\theta}} +  \left(1\vee 64\alpha^{1/\theta}\right) \left(\frac{4LGC(\varepsilon)\varepsilon}{n}\right)^{\frac{1}{2 - \theta}} + \left(1\vee \frac{\alpha^{1/\theta}}{4\epsilon}\right) \left(\frac{4LC(\varepsilon)P_*}{n}\right)^{\frac{1}{2 - \theta}} 
\end{align*}
	\end{proof}
\subsection{Proof of Inequality (\ref{refined:eq1})}
\begin{proof}
	If $\|\wh-\wh^*\|_2^2\leq(\frac{1}{n})^{\frac{\theta}{2-\theta}}$, then $\frac{LC(\varepsilon)\|\wh-\wh^*\|_2^2}{n}\leq  LC(\varepsilon)(\frac{1}{n})^{\frac{2}{2-\theta}}$.
	If $\|\wh-\wh^*\|_2^2\geq(\frac{1}{n})^{\frac{\theta}{2-\theta}}$, then 
	\begin{equation}
	\label{refined:1}
	\frac{1}{\|\wh-\wh^*\|_2^{\frac{2}{\theta}-2}}\leq n^{\frac{1-\theta}{2-\theta}},
	\end{equation}
	so when $ n\geq \left(LC(\varepsilon)\alpha^{\frac{1}{\theta}}/\epsilon\right)^{(2-\theta)}$, we have
	\begin{align*}
	\frac{LC(\varepsilon)\|\wh-\wh^*\|_2^2}{n}&=\frac{LC(\varepsilon)\|\wh-\wh^*\|_2^{\frac{2}{\theta}}\|\wh-\wh^*\|_2^{2-\frac{2}{\theta}}}{n}\\
	&\leq\frac{LC(\varepsilon)\alpha^{\frac{1}{\theta}}(P(\wh)-P_*)}{n^{\frac{1}{2-\theta}}}\leq \epsilon(P(\wh)-P_*),
	\end{align*}
	where the first inequality holds by employing the EBC and the inequality (\ref{refined:1}), and the second inequality holds due to the fact that $n\geq \left(LC(\varepsilon)\alpha^{\frac{1}{\theta}}/\epsilon\right)^{(2-\theta)}$.
	Combining two cases together, we complete the proof.
\end{proof}
\subsection{Proof of Inequality (\ref{refined:eq2})}
\begin{proof}
	The first inequality in the inequality (\ref{refined:eq2}) obviously holds, and now we prove the second inequality.
	\begin{itemize}
		\item If $\|\wh-\wh^*\|_2^2\leq 4(\frac{1}{n})^{\frac{\theta}{2-\theta}}$, then
		\begin{equation*}
		\frac{LC(\varepsilon)\|\wh-\wh^*\|_2^2}{4\epsilon n}\leq\frac{LC(\varepsilon)}{\epsilon}(\frac{1}{n})^{\frac{2}{2-\theta}}.
		\end{equation*}
		\item If $\|\wh-\wh^*\|_2^2\geq 4(\frac{1}{n})^{\frac{\theta}{2-\theta}}$, then
		\begin{equation}
		\label{refined:2}
		\frac{1}{\|\wh-\wh^*\|_2^{2-\frac{2}{\theta}}}\geq\frac{1}{2^{2-\frac{2}{\theta}}}n^{\frac{\theta-1}{2-\theta}}\geq\frac{1}{4}n^{\frac{\theta-1}{2-\theta}},
		\end{equation}
		so when $n\geq \left(LC(\varepsilon)\alpha^{\frac{1}{\theta}}/\epsilon^2\right)^{(2-\theta)}$, we have
		\begin{align*}
		\frac{LC(\varepsilon)\|\wh-\wh^*\|_2^2}{4\epsilon n}
		&=\frac{LC(\varepsilon)\|\wh-\wh^*\|_2^{\frac{2}{\theta}}\|\wh-\wh_*\|_2^{2-\frac{2}{\theta}}}{4\epsilon n}\\
		&\leq \frac{LC(\varepsilon)\alpha^{\frac{1}{\theta}}(P(\wh)-P_*)4n^{\frac{1-\theta}{2-\theta}}}{4\epsilon n}
		\leq\epsilon(P(\wh)-P_*),
		\end{align*}
		where the first inequality holds by employing the EBC and the inequality (\ref{refined:2}), and the second inequality holds due to the fact that $n\geq \left(LC(\varepsilon)\alpha^{\frac{1}{\theta}}/\epsilon^2\right)^{(2-\theta)}$.
	\end{itemize}
	Combining two cases together, we complete the proof.
\end{proof}
	\subsection{Proof of Inequalities (\ref{refined:eq3})--(\ref{refined:eq7})}
	\begin{proof}
		In Lemma \ref{lem:useEBC}, taking $A$ to be
		\begin{equation*}
			\left(\frac{GC(\varepsilon)}{n}\right)^2, \left(\frac{\varepsilon}{n}\right)^2,\varepsilon^2,\frac{4LC(\varepsilon)P_*}{n},\frac{4LGC(\varepsilon)\varepsilon}{n}
		\end{equation*}
		yields inequalities (\ref{refined:eq3})--(\ref{refined:eq7}) respectively.
	\end{proof}
\end{proof}

\section{Proof of Lemma~\ref{lem:net}}
\begin{lem} \citep{Smale:learning}. \label{lem:con} Let $\H$ be a Hilbert space and let $\xi$ be a random variable with values in $\H$. Assume $\|\xi\|\leq G < \infty$ almost surely. Denote $\sigma^2(\xi)=\E\left[\|\xi\|^2\right]$. Let  $\{\xi_i\}_{i=1}^m$ be $m$ ($m < \infty$) independent drawers of $\xi$. For any $0 < \delta < 1$, with confidence $1-\delta$,
	\[
	\left\| \frac{1}{m} \sum_{i=1}^m \left[\xi_i -\E[\xi_i]\right] \right\| \leq \frac{2 G \log(2/\delta)}{m} + \sqrt{\frac{2 \sigma^2(\xi) \log(2/\delta)}{m}}.
	\]
\end{lem}
\begin{proof}[Proof of  Lemma~\ref{lem:net}]
	In order to prove the high probability bounds for all $\w\in\W$, we first consider the points in the $\varepsilon$-net of $\W$ 
	with minimal cardinality. To this end, let $\N(\W, \varepsilon)$ 
	denote the $\varepsilon$-net of $\W$ 
	with minimal cardinality. 
	Since 
	$\W\subseteq\mathcal B^d(R)$, where $\B^d(R)$ denotes a $d$-dimentional  bounded ball with radius $R$. Following the standard results of covering numbers,  we have
	\begin{align*}
	\log |\N(\W, \varepsilon)|\leq \log|\N(\B^d(R), \varepsilon/2)| \leq d \log\frac{6R}{\epsilon}.
	\end{align*}
	
	We first consider a fixed $\w \in \N(\W , \varepsilon)$. 
	Denote by $\w^*$ the closest optimal solution to $\w$.  Let $f_i(\w) = f(\w, \z_i)$. Since $f_i(\cdot)$ is $L$-smooth, we have
	\begin{equation} \label{eqn:smooth:lemma}
	\left\| \nabla f_i(\w) - \nabla f_i (\w^*) \right\|_2 \leq 
	L \|\w - \w^*\|_2.
	\end{equation}
	Because $f_i(\cdot)$ is both convex and $L$-smooth, by (2.1.7) of \citep{nesterov2004introductory}, we have
	\[
	\left\| \nabla f_i(\w) - \nabla f_i (\w^*) \right\|_2^2  \leq L \left(f_i (\w)-  f_i(\w^*)  - \langle \nabla f_i(\w_*), \w-\w^* \rangle \right).
	\]
	Taking expectation over both sides, we have
	\[
	\E \left[ \left\| \nabla f_i(\w) - \nabla f_i (\w^*) \right\|_2^2\right]
	\leq  L \left(F(\w)-  F(\w^*)  - \langle \nabla F(\w^*), \w-\w^* \rangle \right) \leq L \left(P(\w)-  P(\w^*)  \right)
	\]
	where the last inequality follows from the optimality condition of $\w^*$, i.e., there exists $\v_*\in\partial R(\w^*)$
	\[
	\langle \nabla F(\w^*) +\v_*  , \w - \w_* \rangle \geq 0, \ \forall \w \in \W .
	\]
	and the convexity of $R(\w)$ and $F(\w)$, i.e., $\langle \nabla F(\w^*), \w - \w^* \rangle\leq F(\w) - F(\w^*) $ and $\langle \v_*, \w - \w^* \rangle\leq R(\w) - R(\w^*) $. 
	
	Following Lemma~\ref{lem:con},  with probability at least $1-\delta$, we have
	\[
	\begin{split}
	& \left\| \nabla F(\w) - \nabla F(\w^*) - \frac{1}{n} \sum_{i=1}^n [ \nabla f_i(\w) - \nabla f_i (\w^*)] \right\|_2 \\
	& \leq  \frac{2 L \|\w - \w^*\|_2 \log(2/\delta)}{n} + \sqrt{\frac{2 L(P(\w)-  P(\w^*) ) \log(2/\delta)}{n}}.
	\end{split}
	\]
	By taking the union bound over $\N(\W , \varepsilon)$, 
	we have for any $\w\in\N(\W, \varepsilon)$, 
	with  probability $1-\delta$, 
	
	\[
	\begin{split}
	& \left\| \nabla P(\w) - \nabla P(\w^*) - [\nabla P_n(\w) - \nabla P_n(\w^*)] \right\|_2 \\
	&= \left\| \nabla F(\w) - \nabla F(\w^*) - \frac{1}{n} \sum_{i=1}^n [ \nabla f_i(\w) - \nabla f_i (\w^*)] \right\|_2 \\
	& \leq  \frac{2 L \|\w - \w_*\|_2 (\log(2/\delta)+d\log(6R/\varepsilon))}{n} + \sqrt{\frac{2 L(P(\w)-  P(\w^*) )( \log(2/\delta)+d\log(6R/\varepsilon))}{n}}.
	\end{split}
	\]
	To finish the proof of Lemma~\ref{lem:net}, for any $\w\in\W$. There exists $\wt\in\N(\W, \varepsilon)$  such that $\|\w - \wt\|\leq \varepsilon$. Let $\wt^*$ denote the closest optimal solution to $\wt$. Then by non-expansiveness of projection onto a convex set we have $\|\w^* - \wt^*\|_2\leq \|\w - \wt\|_2\leq \varepsilon$.  In addition, we have
	\begin{align}
	\|\wt - \wt^* \|_2&\leq \|\wt - \w\|_2 + \|\w - \w^*\|_2 + \|\w^* - \wt^*\|_2 \leq 2\varepsilon + \|\w - \w^*\|_2\\
	P(\wt) - P(\wt^*)&\leq P(\wt) - P(\w) + P(\w) - P(\w^*) + P(\w^*) - P(\wt^*)\\
	&\leq \bar G\|\wt - \w\|_2 + P(\w) - P(\w^*) + \bar G\|\w^* - \wt^*\|_2\leq 2\bar G\varepsilon+ P(\w) - P(\w^*)\notag
	\end{align}
	Then with probability $1-\delta$, we have
	\begin{align*}
	&\left\| \nabla P(\w) - \nabla P(\w^*) - [\nabla P_n(\w) - \nabla P_n(\w^*)] \right\|_2 \\
	&\leq \left\| \nabla P(\wt) - \nabla P(\wt^*) - [\nabla P_n(\wt) - \nabla P_n(\wt^*)] \right\|_2  + 2L\|\w - \wt\|_2 + 2L\|\w^*- \wt^*\|_2\\
	& \leq  \frac{2 L \|\wt - \wt^*\|_2 (\log(2/\delta)+2d\log(6R/\varepsilon))}{n} \\
	&+ \sqrt{\frac{2 L(P(\wt)-  P(\wt^*) )( \log(2/\delta)+2d\log(6R/\varepsilon))}{n}} + 4L\varepsilon\\
	& \leq  \frac{2 L (\|\w - \w^*\|_2+2\varepsilon) (\log(2/\delta)+2d\log(6R/\varepsilon))}{n} \\
	&+ \sqrt{\frac{2 L(2\bar G\varepsilon+(P(\w)-  P(\w^*)) )( \log(2/\delta)+2d\log(6R/\varepsilon))}{n}} + 4L\varepsilon\\
	&\leq \frac{LC(\varepsilon) \|\w - \w^*\|_2 }{n}  + \frac{2LC(\varepsilon)\varepsilon }{n} + \sqrt{\frac{LC(\varepsilon) (P(\w)-  P_* )}{n} }     + 2\sqrt{\frac{ L\bar GC(\varepsilon) \varepsilon }{n}} 
	+ 4L\varepsilon.
	\end{align*}
\end{proof}

\section{Proof of Lemma~\ref{lem:vec:con}}
\begin{proof}
	We first consider a fixed $\w_*\in\N(\W_*, \varepsilon)\subseteq\W_*$. To apply Lemma~\ref{lem:con}, we need an upper bound of $\E\left[\|\nabla f_i (\w_*)\|_2^2\right]$. Since $f_i(\cdot)$ is $L$-smooth and nonnegative, from Lemma 4.1 of \citep{Smooth:Risk}, we have
	\[
	\|\nabla f_i(\w_*)\|_2^2 \leq 4 L  f_i (\w_*)
	\]
	and thus
	\[
	\E\left[\|\nabla f_i (\w_*)\|_2^2\right] \leq 4 L  \E\left[ f_i (\w_*) \right]=  4 L F_*.
	\]
	By {\bf Assumption~\ref{ass:1}}, we have $\|\nabla f_i (\w_*)\|_2\leq G $. Then, according to Lemma~\ref{lem:con}, with probability at least $1-\delta$, we have
	\[
	\begin{split}
	\left\|\nabla F(\w_*) - \nabla F_n(\w_*)\right\|_2 = \left\|\nabla F(\w_*) - \frac{1}{n}\sum_{i=1}^n \nabla f_i(\w_*) \right\|_2\\\leq  \frac{2G\log(2/\delta)}{n} + \sqrt{\frac{8 L F_* \log(2/\delta)}{n}}.
	\end{split}
	\]
	By taking the union bound over $\N(\W_*, \varepsilon)$, for any $\w_*\in \N(\W_*, \varepsilon)$, with probability $1-\delta$ we have
	\[
	\left\|\nabla F(\w_*) - \nabla F_n(\w_*)\right\|_2 \leq  \frac{GC(\varepsilon)}{n} + \sqrt{\frac{4 L F_*C(\varepsilon)}{n}}.
	\]
	For any $\w^*\in\W_*$, there exists $\wt^*\in\N(\W_*, \varepsilon)$ such that $\|\w^* - \wt^*\|\leq \varepsilon$. Then 
	\begin{align*}
		& \left\|\nabla F(\w^*) - \nabla F_n(\w^*)\right\|_2 \\& \leq \left\|\nabla F(\wt^*) - \nabla F_n(\wt^*)\right\|_2+  \left\|\nabla F(\w^*) - \nabla F(\wt^*)\right\|_2  \\&+  \left\|\nabla F_n(\w^*) - \nabla F_n(\wt^*)\right\|_2\\
		&\leq \frac{GC(\varepsilon)}{n} + \sqrt{\frac{4 L F_*C(\varepsilon)}{n}} + 2L\varepsilon.
	\end{align*}
	
\end{proof}
\section{Proof of Lemma~\ref{lem:useEBC}}
\begin{proof}
	We consider two cases. First, $\|\w - \w^*\|_2\leq A^{\frac{\theta}{4-2\theta}}$, under which the inequality follows trivially.  Next, we consider $\|\w - \wh^*\|_2\geq  A^{\frac{\theta}{4-2\theta}}$. Then 
	\begin{align*}
		&\left\| \w - \w^* \right\|_2 \sqrt{A} =\frac{ \|\w - \w^*\|_2^{1/\theta}}{\|\w - \w^*\|_2^{1/\theta -1}}\sqrt{A}\\
		&\leq \|\w - \w^*\|_2^{1/\theta}A^{\frac{1}{2(2-\theta)}}\leq\frac{\epsilon\|\w - \w^*\|_2^{2/\theta} }{\alpha^{1/\theta}}+ \frac{\alpha^{1/\theta}}{4\epsilon}A^{\frac{1}{2-\theta}}\\
		&\leq \epsilon(P(\w) - P_*) + \frac{\alpha^{1/\theta}}{4\epsilon}A^{\frac{1}{2-\theta}}
	\end{align*} 
	where the last inequality follows the EBC. 
\end{proof}

\section{Proof of Theorem~\ref{thm:high}}
Before proceeding to the proof, we first present a standard result for SSG, which is the Lemma 10 of \citep{hazan-20110-beyond}.
\begin{prop}\label{prop2}
	Suppose {\bf Assumptions~\ref{ass:1} and \ref{ass:2}} hold.  Let $0<\delta<1$, $\w^*\in\mathcal{W}_*$ be the closest optimal solution to $\w_1$, and $R_0$ be an upper bound on $\|\w_1 - \w^*\|_2$. Apply T iterations  of the update $\w_{t+1} =\Pi_{\W\cap \B(\w_1,R_0)}(\w_t -\gamma g_t)$, where $g_t$ is a stochastic subgradient of $P(\w)$ at $\w_t$. With probability at least $1-\delta$, we have
	\begin{equation*}
		P(\wh_T)-P_*\leq\frac{\gamma G^2}{2}+\frac{\|\w_1-\w^*\|_2^2}{2\gamma (T+1)}+\frac{4GR_0\sqrt{2\log(2/\delta)}}{\sqrt{T+1}}.
	\end{equation*}
	where $\wh_T = \frac{1}{T+1}\sum_{t=1}^{T+1} \w_t$. 
	Moreover, choose $\gamma=\frac{R_0}{G\sqrt{T+1}}$, and then with probability at least $1-\delta$,
	\begin{equation*}
		P(\wh_T)-P_*\leq R_0G\left(\frac{1}{\sqrt{T+1}}+\frac{4\sqrt{2\log(2/\delta)}}{\sqrt{T+1}}\right).
	\end{equation*}
\end{prop}
It is easy to derive a similar lemma as Proposition \ref{prop2}, which is stated in Lemma \ref{lemma:new}.

\begin{lem}
	\label{lemma:new}
	Suppose {\bf Assumptions~\ref{ass:1},~\ref{ass:2}} hold. Let $0<\delta<1$, $R_0$ be any nonnegative real number. Apply T iterations of the update $\w_{t+1} =\Pi_{\W\cap \B(\w_1,R_0)}(\w_t -\gamma g_t)$, where $g_t$ is a stochastic subgradient of $P(\w)$ at $\w_t$. With probablity at least $1-\delta$, we have
	\begin{align*}
	P(\wh_T)-P(\w_1)\leq \frac{\gamma G^2}{2}+\frac{4GR_0\sqrt{2\log(2/\delta)}}{\sqrt{T+1}},
	\end{align*}
	where $\wh_T = \frac{1}{T+1}\sum_{t=1}^{T+1} \w_t$. 
	Moreover, choose $\gamma=\frac{R_0}{G\sqrt{T+1}}$, and then with probability at least $1-\delta$,
	\begin{equation*}
	P(\wh_T)-P(\w_1)\leq R_0G\left(\frac{1}{\sqrt{T+1}}+\frac{4\sqrt{2\log(2/\delta)}}{\sqrt{T+1}}\right).
	\end{equation*}
\end{lem}
\begin{proof}
	Denote $\E_{t-1}(X)$ by the expectation conditioned on the randomness until round $t-1$, then we have $\E_{t-1}(\hat{g}_t)=g_t$, and $X_t=g_t(\w_t-\w_1)-\hat{g}_t(\w_t-\w_1)$ is a martingale difference sequence. Note that $\|g_t\|_2=\|\E_{t-1}(\hat{g}_t)\|_2\leq\E_{t-1}(\|\hat{g}_t\|_2)\leq G$, so we have
	\begin{align*}
	|X_t|\leq\|g_t\|_2\|\w_t-\w_1\|_2+\|\hat{g}_t\|_2\|\w_t-\w_1\|_2\leq 4GR_0,
	\end{align*}
	since the update needs to project the gradient update onto the intersection of $\W$ and a ball with radius $R_0$.
	
	By Azuma-Hoeffding's inequality, we have with probability at least $1-\delta$,
	\begin{align}
	\label{inequality:azuma}
	\frac{1}{T+1}\sum_{t=1}^{T+1}g_t(\w_t-\w_1)-\frac{1}{T+1}\sum_{t=1}^{T}\hat{g}_t(\w_t-\w_1)\leq\frac{4GR_0\sqrt{2\log(1/\delta)}}{\sqrt{T+1}}.
	\end{align}
	By the convexity of $P$, we have $P(\w_{t})-P(\w_1)\leq g_{t}(\w_{t}-\w_1)$, then using a standard result in online gradient descent \citep{zinkevich2003online}, we have
	\begin{align}
	\label{inequality:online}
	\frac{1}{T+1}\sum_{t=1}^{T}\hat{g}_t(\w_t-\w_1)\leq\frac{\gamma G^2}{2}+\frac{\|\w_1-\w_1\|_2^2}{2\gamma(T+1)}=\frac{\gamma G^2}{2}.
	\end{align}
	Combining inequality (\ref{inequality:azuma}) and (\ref{inequality:online}) suffices to derive the conclusion.
\end{proof}
With the above proposition and lemma, the proof of Theorem~\ref{thm:high} proceeds similarly as that of Theorem 5.3 in~\citep{juditsky2014}. The difference is that our analysis only relies on the EBC instead of the uniform convexity. 
\begin{proof}
	Define $\bar{\delta}=\frac{2\delta}{\log_2 n}$, and 
	\begin{equation*}
		a(n,\bar{\delta})=G\bigg(\frac{1}{\sqrt{n+1}}+\frac{4\sqrt{2\log(2/\bar{\delta})}}{\sqrt{n+1}}\bigg).
	\end{equation*}
	We set $\mu_0=2R_0^{1-\frac{2}{\theta}}a(n_0,\bar{\delta})$, $\mu_k=2^{(\frac{2}{\theta}-1)k}\mu_0$ and $R_k = R_0/2^k$, where $k=1,\ldots,m$. Then we have $\mu_kR_k^{\frac{2}{\theta}} = 2^{-k}\mu_0R_0^{\frac{2}{\theta}}$. 
	We can also assume that $\alpha$ is large enough such that $\alpha\geq R_0^{2-\theta}/G^{\theta}$,  i.e., $\alpha^{-\frac{1}{\theta}}\leq GR_0^{1-\frac{2}{\theta}}$, otherwise we can set $\alpha =  R_0^{2-\theta}/G^{\theta}$, which makes the EBC still hold. 
	

	By definition of $m$, when $n\geq 100$,
	\begin{equation}
		\label{eq:m}
		0<\frac{1}{2}\log_2\frac{2n}{\log_2 n}-2\leq m\leq \frac{1}{2}\log_2\frac{2n}{\log_2 n}-1\leq\frac{1}{2}\log_2 n,
	\end{equation}
	so we have 
	\begin{equation}
		\label{eq:2m}
		2^m\geq\frac{1}{4}\sqrt{\frac{2n}{\log_2 n}}.
	\end{equation}
	When $n\geq 100$, we have 
	\begin{align*}
		\mu_m &=2^{(\frac{2}{\theta}-1)m}\mu_0\geq 2^m\mu_0 \\
		&\geq \frac{1}{4}\sqrt{\frac{2n}{\log_2 n}} 4GR_0^{1-\frac{2}{\theta}}\left(\frac{1}{2\sqrt{n_0+1}}+\frac{2\sqrt{2\log(\log_2 n)}}{\sqrt{n_0+1}}\right)\\
		&\geq GR_0^{1-\frac{2}{\theta}}\sqrt{\frac{2n}{\log_2 n}}\left(\frac{1}{2\sqrt{\frac{n}{m}+1}}+\frac{2\sqrt{2\log(\log_2 n)}}{\sqrt{\frac{n}{m}+1}}\right)\\
		&\geq GR_0^{1-\frac{2}{\theta}}\sqrt{\frac{2n}{\log_2 n}}\left(\frac{1}{2\sqrt{\frac{2n}{\log_2 2n-\log_2\log_2 n-4}+1}}+\frac{2\sqrt{2\log(\log_2 n)}}{\sqrt{\frac{2n}{\log_2 2n-\log_2\log_2 n-4}+1}}\right)\\
		&\geq GR_0^{1-\frac{2}{\theta}}\sqrt{\frac{2n}{\log_2 n}}\frac{2\sqrt{\sqrt{2\log(\log_2 n)}}}{\sqrt{\frac{2n}{\log_2 2n-\log_2\log_2 n-4}+1}}\\
		&= GR_0^{1-\frac{2}{\theta}}\frac{2\sqrt{\sqrt{2\log(\log_2 n)}}}{\sqrt{\frac{1}{1-\frac{\log_2\log_2 n + 3}{\log_2 n}}+\frac{\log_2 n}{2n}}}\geq GR_0^{1-\frac{2}{\theta}},
	\end{align*}
	where the first inequality holds because $\theta\in(0,1]$, the second inequality comes from (\ref{eq:2m}) and the fact that $0<\delta<1$, the third and fourth inequalities hold because of the definition of $n_0$ and inequality (\ref{eq:m}), the fifth inequality holds by utilizing $a+b\geq 2\sqrt{ab}$, and the sixth inequality holds since $n\geq 100$ and the function is monotonically increasing with respect to $n$.  So $\alpha^{-\frac{1}{\theta}}\leq\mu_m$.

	Below, given $\wh_{k}$ we denote by $\wh^*_k$ the closest optimal solution to $\wh_k$. Next, we consider two cases.
	\paragraph{Case 1.}	If $\alpha^{-\frac{1}{\theta}}\geq \mu_0$, then $\mu_0\leq \alpha^{-\frac{1}{\theta}}\leq\mu_m$. We have the following lemma.
	\begin{lem}
		Let $k^*$ satisfy $\mu_{k^*}\leq \alpha^{-\frac{1}{\theta}}\leq 2^{\frac{2}{\theta}-1}\mu_{k^*}$. Then for any $1\leq k\leq k^*$, there exists a Borel set $\mathcal{A}_k\subset\Omega$ of probability at least $1-k\bar{\delta}$, such that for $\omega\in\mathcal{A}_k$, the points $\{\wh_k\}_{k=1}^{m}$ generated by the Algorithm \ref{alg:FSA} satisfy
		\begin{align}
			\label{high:eq1}
			\|\wh_{k-1}-\wh^*_{k-1}\|_2\leq R_{k-1}=2^{-k+1}R_0,\\
			\label{high:eq2}
			P(\wh_k)-P_*\leq\mu_kR_k^{\frac{2}{\theta}}=2^{-k}\mu_0R_0^{\frac{2}{\theta}}.
		\end{align}
		Moreover, for $k>k^*$ there is a Borel set $\mathcal{C}_k\subset \Omega$ of probability at least $1-(k-k^*)\bar{\delta}$ such that on $\mathcal{C}_k$, we have
		\begin{equation}
			\label{high:eq3}
			P(\wh_k)- P(\wh_{k^*})\leq \mu_{k^*}R_{k^*}^{\frac{2}{\theta}}.
		\end{equation}
	\end{lem}
	\begin{proof}
		We prove (\ref{high:eq1}) and~(\ref{high:eq2}) by induction. Note that (\ref{high:eq1}) holds for $k=1$. Assume it is true for some $k> 1$ on $\mathcal{A}_{k-1}$. According to the Proposition \ref{prop2}, there exists a Borel set $\mathcal{B}_k$ with $\text{Pr}(\mathcal{B}_k)\geq 1-\bar{\delta}$ such that
		\begin{align*}
			P(\wh_{k})-P_*&\leq R_{k-1}G\left(\frac{1}{\sqrt{n_0+1}}+\frac{4\sqrt{2\log(2/\bar{\delta})}}{\sqrt{n_0+1}}\right)\\
			&=R_{k-1}a(n_0,\bar{\delta})=\frac{1}{2}\mu_k2^{(1-\frac{2}{\theta})k}R_0^{\frac{2}{\theta}-1}R_{k-1}\\&=\mu_kR_k^{\frac{2}{\theta}},
		\end{align*}
		which is~(\ref{high:eq2}). 
		By the inductive hypothesis, $\|\wh_{k-1}-\w^*_{k-1}\|_2\leq R_{k-1}$ on the set $\mathcal{A}_{k-1}$. Define $\mathcal{A}_k=\mathcal{A}_{k-1}\cap\mathcal{B}_k$. Note that 
		\begin{equation*}
			\text{Pr}(\mathcal{A}_k)\geq \text{Pr}(\mathcal{A}_{k-1})+\text{Pr}(\mathcal{B}_k)-1\geq 1-k\bar{\delta},
		\end{equation*}
		and on $\mathcal{A}_k$, by the EBC and the definition of $k^*$, we have
		\begin{align*}
			\|\wh_k-\wh^*_{k}\|_2^{\frac{2}{\theta}}&\leq \alpha^{\frac{1}{\theta}}(P(\wh_k)-P_*)\leq\frac{P(\wh_k)-P_*}{\mu_{k^*}}\\&\leq\frac{\mu_kR_k^{\frac{2}{\theta}}}{\mu_{k^*}}\leq R_k^{\frac{2}{\theta}},
		\end{align*}
		which is (\ref{high:eq1}) for $k+1$.
		
		Now we prove (\ref{high:eq3}). For $k>k^*$, by Lemma \ref{lemma:new}, there exists a Borel set $\mathcal{B}_k$ with $\text{Pr}(\mathcal{B}_k)\geq 1-\bar{\delta}$ such that
		\begin{align*}
			P(\wh_k)-P(\wh_{k-1})&\leq\frac{\gamma_k G^2}{2}+\frac{4GR_{k-1}\sqrt{2\log(2/\delta)}}{\sqrt{n_0+1}}\\
			&\leq R_{k-1}a(n_0,\bar{\delta})\\
			&=2^{k^*-k}R_{k^*-1}a(n_0,\bar\delta)\\
			&=2^{k^*-k}\mu_{k^*}R_{k^*}^{\frac{2}{\theta}}=\mu_kR_k^{\frac{2}{\theta}},
		\end{align*}
		which implies that on $\mathcal{C}_k=\cap_{j=k^*+1}^{k}\mathcal{B}_j$, we have
		\begin{equation*}
		\begin{aligned}
			&P(\wh_k)-P(\wh_{k^*})=\sum_{j=k^*+1}^{k}\left(P(\wh_j)-P(\wh_{j-1})\right)\\&\leq \sum_{j=k^*+1}^{k}2^{k^*-j}\mu_{k^*}R_{k^*}^{\frac{2}{\theta}}\leq\mu_{k^*}R_{k^*}^{\frac{2}{\theta}}.
			\end{aligned}
		\end{equation*}
		By union bound, we have $\text{Pr}(\cap_{j=k^*+1}^{k}\mathcal{B}_j)\geq 1-(k-k^*)\bar{\delta}$.
		Here completes the proof.
		
	\end{proof}
	Now we proceed the proof as follows. Note that $\mu_0\leq \alpha^{-\frac{1}{\theta}}\leq \mu_m$. At the end of $k^*$-th stage, on the Borel set $\mathcal{A}_{k^*}$ of probability at least $1-k^*\bar{\delta}$, we have
	\begin{equation*}
		P(\wh_{k^*})-P_*\leq\mu_{k^*}R_{k^*}^{\frac{2}{\theta}}.
	\end{equation*}
	Then on the Borel set $\mathcal{D}_m=\mathcal{C}_m\cap\mathcal{A}_{k^*}=(\cap_{j=k^*+1}^{m}\mathcal{B}_j)\cap A_{k^*}$  with $\text{Pr}(\mathcal{D}_m)\geq 1-m\bar{\delta}$, we have
	\begin{align*}
		P(\wh_m)-P_*&=P(\wh_m)-P(\wh_{k^*})+(P(\wh_{k^*})-P_*) \\
		&\leq 2\mu_{k^*}R_{k^*}^{\frac{2}{\theta}}\leq 4(\frac{\mu_{k^*}}{\alpha^{-\frac{1}{\theta}}})^{\frac{1}{\frac{2}{\theta}-1}}\mu_{k^*}R_{k^*}^{\frac{2}{\theta}}\\&=4\left(\frac{2^{(\frac{2}{\theta}-1)k^*}\mu_0}{\alpha^{-\frac{1}{\theta}}}\right)^{\frac{1}{\frac{2}{\theta}-1}}\mu_{k^*}R_{k^*}^{\frac{2}{\theta}}\\
		&=4(2^{k^*}\mu_{k^*}R_{k^*}^{\frac{2}{\theta}}\mu_0^{\frac{\theta}{2-\theta}}\alpha^{\frac{1}{2-\theta}})\\&=4(\mu_0R_0^{\frac{2}{\theta}}\mu_0^{\frac{\theta}{2-\theta}}\alpha^{\frac{1}{2-\theta}})\\
		&=4[(2R_0^{1-\frac{2}{\theta}}a(n_0,\bar{\delta}))^{\frac{2}{2-\theta}}R_0^{\frac{2}{\theta}}\alpha^{\frac{1}{2-\theta}}]\\&= 4(2\sqrt{\alpha}\cdot a(n_0,\bar{\delta}))^{\frac{2}{2-\theta}} \\
		&= (2^{2-\theta}2\sqrt{\alpha}\cdot a(n_0,\bar{\delta}))^{\frac{2}{2-\theta}}.			 		
	\end{align*}
	By the definition of $m$ and $\bar{\delta}$, and the fact that $m\leq\frac{1}{2}\log_2 n$, we have $m\bar{\delta}\leq\delta$. So $\text{Pr}(\mathcal{D}_m)\geq 1-\delta$.
	
	\paragraph{Case 2.}
	If $\alpha^{-\frac{1}{\theta}}<\mu_0$, then on $\mathcal{A}_1=\mathcal{B}_1$,
	\begin{align*}
		P(\wh_1)-P_*&\leq R_0\cdot a(n_0,\bar{\delta})=\frac{R_0}{a(n_0,\bar{\delta})^{\frac{\theta}{2-\theta}}}\cdot a(n_0,\bar{\delta})^{\frac{2}{2-\theta}}\\
		&=\frac{2^{\frac{\theta}{2-\theta}}}{\mu_0^{\frac{\theta}{2-\theta}}}a(n_0,\bar{\delta})^{\frac{2}{2-\theta}}\leq 2^{\frac{\theta}{2-\theta}}\bigg(\sqrt{\alpha}\cdot a(n_0,\bar{\delta})\bigg)^{\frac{2}{2-\theta}}.
	\end{align*}
	Hence on $\mathcal{A}_1\cap\mathcal{C}_m$, by a similar argument as in case 1, we have
	\begin{equation*}
	\begin{split}
		P(\wh_m)-P_*=P(\wh_m)-P(\wh_1)+P(\wh_1)-P_*\\
		\leq 2R_0\cdot a(n_0,\bar{\delta})\leq (2\sqrt{\alpha}\cdot a(n_0,\bar{\delta}))^{\frac{2}{2-\theta}},
		\end{split}
	\end{equation*}
	where $\text{Pr}(\mathcal{A}_1\cap\mathcal{C}_m)\geq 1-\delta$.

	Combining the two cases, we have with probability at least $1-\delta$, 
	\begin{align*}
		&P(\wh_m)-P_*\\&\leq (8\sqrt{\alpha} \vee 2\sqrt{\alpha})^{\frac{2}{2-\theta}}\left(G\left(\frac{1}{\sqrt{n_0+1}}+\frac{4\sqrt{2\log(2/\bar{\delta})}}{\sqrt{n_0+1}}\right)\right)^{\frac{2}{2-\theta}}\\
		&\leq (64\alpha)^{\frac{1}{2-\theta}}\left(\frac{G\left(1+4\sqrt{2\log(\frac{\log_2 n}{\delta})}\right)}{\sqrt{\frac{n}{\frac{1}{2}\log_2 n}}}\right)^{\frac{2}{2-\theta}}\\
		&=\left(\frac{128\alpha G^2\log_2 n\left(1+4\sqrt{2\log(\frac{\log_2 n}{\delta})}\right)^2}{n}\right)^{\frac{1}{2-\theta}},
	\end{align*}
	where the second inequality stems from the fact that $n_0+1\geq \frac{n}{m}\geq\frac{n}{\frac{1}{2}\log_2 n}$.
\end{proof}

\section{Detailed Analysis of Examples Satisfying EBC}

\paragraph{Risk Minimization Problems over an $\ell_2$ ball.} 
\begin{lem}\label{lem:2}
Consider the following problem 
\begin{align}
	\min_{\|\w\|_2\leq B}P(\w)\triangleq\E_{\z}[f(\w, \z)]
\end{align}
If $\min_{\w\in\R^d}P(\w)<\min_{\|\w\|_2\leq B}P(\w)$, then the above problem satisfies EBC$(\theta=1, \alpha)$.
\end{lem}
\begin{proof}
The proof is similar to that of Theorem 3.5 of \citep{guoyincalculus2016}. Denote $\w_*$ by an optimal solution of Example 4. Let $\Omega=\{\w\in\R^d\;|\;\|\w\|_2\leq B\}$, and $F(\w)=P(\w)+I_{\Omega}(\w)$, where $I_{\Omega}(\w)=0$ if $\w\in\Omega$, and otherwise $I_{\Omega}(\w)=+\infty$. Then we have $\arg\min_{\w\in\R^d}F(\w) = \arg\min_{\|\w\|_2\leq B}P(\w)$. Let $\w_*\in\arg\min_{\w\in\R^d}F(\w)$ denote an optimal solution.  

Since $B>0$, so the optimization problem is strictly feasible, then by the Lagrangian theory,  there exists some $\lambda\geq 0$, such that
\begin{equation*}
\begin{aligned}
	F(\w_*)&=\min_{\|\w\|_2\leq B}P(\w)=\min_{\w\in\R^d}(P(\w)+\lambda(\|\w\|_2^2-B^2))\\
	&=P(\w_*)+\lambda(\|\w_*\|_2^2-B^2).
\end{aligned}
\end{equation*}
Note that $\min_{\w\in\R^d}P(\w)<\min_{\|\w\|_2\leq B}P(\w)$, as a result $\lambda>0$. Then by complementary slackness, we know that $\|\w_*\|_2=B$. Denote by $P_\lambda(\w)=P(\w)+\lambda(\|\w\|_2^2-B^2)$. Then according to Theorem 28.1~\citep{rockafellar1970convex}, we have 
\begin{align}
	\label{ex4:1}
	\w_*\in \arg\min F&=\{\w\;|\;\|\w\|_2=B\}\cap\arg\min_{\w\in\R^d}P_\lambda(\w). 
\end{align}
Since $P_{\lambda}(\w)$ is strongly convex due to $\lambda>0$, its optimal solution is unique. As a result, 
\begin{align}
	\label{ex4:1}
	\w_*=\arg\min F&=\arg\min_{\w\in\R^d}P_\lambda(\w). 
\end{align}
In addition, there exists $\mu>0$ such that (due to the strong convexity of $P_\lambda(\w)$), 
\begin{align*}
	&\|\w-\arg\min P_{\lambda}(\w)\|_2\leq\mu(P_{\lambda}(\w)-\min_{\w}P_{\lambda}(\w))^{1/2}\\
	&=\mu(P(\w)+\lambda(\|\w\|_2^2-B^2)-P(\w_*))^{1/2}\\
	&\leq \mu(P(\w)-P(\w_*))^{1/2}.
\end{align*}
Then according to (\ref{ex4:1}), we know that
\begin{equation*}
	\|\w-\w_*\|_2^2\leq\mu^2(P(\w)-P(\w_*)),
\end{equation*}
which is EBC$(\theta=1,\mu^2)$.
\end{proof}

\paragraph{Quadratic Problems.}
\begin{lem}\label{lem:3}
Consider the following problem 
\begin{equation}
	\min_{\w\in\W}P(\w)\triangleq  \w^{\top}\E_{\z}[A(\z)]\w  +  \w^{\top}\E_{\z'}[\b(\z')] + c
\end{equation}
If $\E_\z[A(\z)]$ is PSD and $\W$ is a bounded polyhedron, then the above problem satisfies EBC$(\theta=1, \alpha)$.
\end{lem}
\begin{proof}
Let us consider $\E_{\z}[A(\z)]\neq 0$; otherwise it reduces to PLP. 

Note that $\E_{\z}[A(\z)]$ is PSD, so there exists a nonzero matrix $A$ such that $\E_{\z}[A(\z)]=A^\top A$. The original optimization problem is equivalent to
\begin{equation}
	\min_{\w\in\W}g(A\w)+\w^{\top}\E_{\z'}[b(\z')]+c,
\end{equation}
where $g(\u)=\u^\top\u$ is a strongly convex function of $\u$. Since the constraint is a polyhedral function of $\w$,  according to the Lemma 12 of~\citep{DBLP:journals/corr/arXiv:1512.03107}, we know that the optimization problem satisfies EBC$(\theta=1,\alpha)$.
\end{proof}

\paragraph{Piecewise Linear Problems (PLP)}
\begin{lem}\label{lem:4}
Consider the problem
\begin{align}\label{eqn:plp2}
	\min_{\w\in\W}P(\w)\triangleq    \E[f(\w, \z)]
\end{align}
where   $\E[f(\w, \z)]$ is a piecewise linear function and $\W$ is a bounded polyhedron. Then the problem~(\ref{eqn:plp2})  satisfies EBC$(\theta=1, \alpha)$.
\end{lem}
\begin{proof}
According to weak sharp minima condition~\citep{doi:10.1137/0331063} (e.g., Lemma 8 in~\citep{DBLP:journals/corr/arXiv:1512.03107}), we have
\[
\|\w - \w^*\|_2^2\leq c(P(\w) - P(\w^*))^2, 
\]
Since $P(\w)$ is piecewise linear, then $P(\w) - P(\w_*)$ is bounded on a bounded set. Then there exists $\alpha>0$ such that 
\[
\|\w - \w^*\|_2^2\leq \alpha (P(\w) - P(\w^*)), 
\]

\end{proof}

\paragraph{$\ell_1$ regularized problems}
\begin{lem}\label{lem:5}
Consider the problem: for $\ell_1$ regularized risk minimization: 
\begin{align}\label{eqn:l1r2}
	\min_{\|\w\|_1\leq B}  F(\w)\triangleq P(\w)+ \lambda \|\w\|_1,
\end{align} 
If  $P(\w)$ is convex quadratic or piecewise linear, then the problem~(\ref{eqn:l1r2})  satisfies EBC$(\theta=1, \alpha)$. \end{lem}
\begin{proof}
 It is easy to see that $P(\w)$ is either piecewise linear or piecewise convex quadratic. According to Lemma 3.3 of~\citep{DBLP:journals/mp/Li13}, we have
\begin{itemize}
	\item When $P(\w)$ is piecewise linear, there exists $\alpha_1, \alpha>0$, such that
	\begin{align*}
		\|\w-\w^*\|_2^2&\leq \alpha_1(P(\w)-P(\w^*))^2\\
		&\leq \alpha(P(\w)-P(\w^*)),
	\end{align*}
	where we use the fact $P(\w) - P(\w_*)$ is bounded over a bounded domain due to its Lipschitz continuity. 
	\item When $P(\w)$ is piecewise convex quadratic, there exists $\alpha_2>0$, such that
	\begin{equation*}
		\|\w-\w^*\|_2^2\leq\alpha_2(P(\w)-P(\w^*)).
	\end{equation*}
\end{itemize}
\end{proof}

\begin{lem}\label{lem:6}
Consider the problem:
\begin{align}
	\min_{\w\in\W}F(\w)\triangleq  P(\w)+ \lambda\|\w\|_p^p
\end{align}
If $P(\w)$  is convex quadratic, and $\W$ is a bounded polyheron, then the above problem satisfies EBC$(\theta=1/p, \alpha)$.
\end{lem}
\begin{proof}
According to Theorem 5.2~\citep{doi:10.1137/070689838}, the objective function is $p$-th order convex polynomial function and  $\forall \w\in\W$ there exists $\tau>0$ such that
\begin{align*}
	\|\w - \w^*\|_2\leq \tau (P(\w) -P(\w^*)  + (P(\w) - P(\w^*))^{1/p}).
\end{align*}
There exists $c>0$ such that $P(\w) - P(\w^*)\leq c$ for any $\w\in\W$. Then
\begin{align*}
	&\|\w - \w^*\|_2\leq \tau(c^{1-1/p}+1) (P(\w) - P(\w^*))^{1/p},
\end{align*}
i.e., 
\begin{align*}
	&\|\w - \w^*\|^2_2\leq \tau^2(c^{1-1/p}+1)^2 (P(\w) - P(\w^*))^{2/p}.
\end{align*}
\end{proof}

\section{Proof of Corollary~\ref{cor:3}}
The objective function is a semi-algebraic function. As a result, there must exists $\theta\in(0,2]$ such that EBC holds according to existing results~\citep{arxiv:1510.08234}. If $\theta>1$, then EBC also holds with $\theta=1$ due to that the objective function is bounded. 


\section{Different Variants of ASA}
In this section, we introduce two variants of ASA, namely ASA2 (Algorithm \ref{alg:FSA2}) and ASA3 (Algorithm \ref{alg:ASA3}). Compared with ASA, ASA2 can get around of the expensive projection operation and ASA3 can allow a subroutine with proximal mapping.
\subsection{A variant of ASA without projection}
Now we provide a different variant of ASA, which utilizes SSGS (Algorithm 2 in~\citep{DBLP:journals/corr/abs-1607-01027}) as a subroutine to avoid the projection onto the intersection of $\W$ and a bounded ball in the vanilla ASA. SSGS is an algorithm which adds a strongly convex regualarizer to the original loss function, i.e.,
\begin{equation*}
	\min_{\w\in\W}P(\w)+\frac{1}{2\beta}\|\w-\w_1\|_2^2,
\end{equation*}
where $\w_1\in \W$ is called reference point. For completeness, we describe the SSGS and the corresponding ASA2 algorithms in Algorithm \ref{alg:ssgs} and Algorithm \ref{alg:FSA2} respectively.

		\begin{algorithm}[t]
			\caption{SSGS$(\w_1,\beta,T)$}
			\label{alg:ssgs}
			\begin{algorithmic}[1]
				\REQUIRE~~$\w_1\in \mathcal{W}$, $\beta>0$ and  $T$
				\ENSURE~~$\wh_T$
				\FOR {$t=1,\ldots, T$}
				\STATE $\w'_{t+1}=(1-\frac{2}{t})\w_t+\frac{2}{t}\w_1-\frac{2\beta}{t}g_t$
				\STATE $\w_{t+1}=\Pi_{\W}(\w'_{t+1})$
				\ENDFOR
				\STATE $\wh_T=\frac{1}{T+1}\sum_{t=1}^{T+1}{\w_t}$
				\STATE return $\wh_T$
			\end{algorithmic}
		\end{algorithm}
	%
		\begin{algorithm}[t]
			\caption{ASA2($\w_1,n,R_0$)}
			\label{alg:FSA2}
			\begin{algorithmic}[1]
				\REQUIRE~~$\w_1\in \W$, $n$ and  
				$R_0=2R$
				\ENSURE~~$\wh_m$
				\STATE Set $\wh_0=\w_1$, $m=\lfloor \frac{1}{2}\log_2\frac{2n}{\log_2 n}\rfloor-1$, $n_0=\lfloor n/m\rfloor$ 
				\FOR{$k=1,\ldots,m$}
				\STATE Set $\beta_k=\frac{R_{k-1}\sqrt{n_0}}{2G}$ and $R_k = R_{k-1}/2$
				\STATE $\wh_{k}=\text{SSGS}(\wh_{k-1},\beta_k,n_0)$
				\ENDFOR
			\end{algorithmic}
		\end{algorithm}
We first present a result for analyzing SSGS, which is the Corollary 5 in \citep{DBLP:journals/corr/abs-1607-01027}.
\begin{prop}\label{prop:ssgs}
	Suppose {\bf Assumptions~\ref{ass:1} and \ref{ass:2}} hold.  Let $0<\delta<1/e$, $T\geq 3$, $\w^*\in\mathcal{W}_*$ be the closest optimal solution to $\w_1$, and $R_0$ be an upper bound on $\|\w_1 - \w^*\|_2$. Apply T iterations of the SSGS (Algorithm \ref{alg:ssgs}) and return the average solution, where $g_t$ is a stochastic subgradient of $P(\w)$ at $\w_t$. With probability at least $1-\delta$, we have
	\begin{align*}
		&P(\wh_T)-P_*\leq\frac{1}{2\beta}\|\w_1-\w^*\|_2^2+\frac{34\beta G^2(1+\log T+\log(4\log T/\delta))}{T}.
	\end{align*}
	where $\wh_T = \frac{1}{T+1}\sum_{t=1}^{T+1} \w_t$. Moreover, choose $\beta=\frac{R_0\sqrt{T}}{2G}$, and then with probability at least $1-\delta$, 
	\begin{align*}
		&P(\wh_T)-P_*
		\leq R_0G\left(\frac{1}{\sqrt{T}}+\frac{17\left(1+\log T+\log\left(4\log T/\delta\right)\right)}{\sqrt{T}}\right).
	\end{align*}
	
	Similarly, for any nonnegative $R_0$, by choosing $\beta=\frac{R_0\sqrt{T}}{2G}$, and then with probability at least $1-\delta$, 
	\begin{align*}
		&P(\wh_T)-P(\w_1)
		\leq R_0G\left(\frac{1}{\sqrt{T}}+\frac{17\left(1+\log T+\log\left(4\log T/\delta\right)\right)}{\sqrt{T}}\right).
	\end{align*}
\end{prop}
Then we provide the high probability analysis of ASA2, which is Theorem \ref{thm:high2}.
\begin{thm}
	\label{thm:high2}
	Suppose {Assumptions~\ref{ass:1}, and~\ref{ass:2}} hold. Let $\wh_m$ be the returned solution of the Algorithm \ref{alg:FSA2}. For  $n\geq 100$ and any $\delta\in(0,1)$, with probability at least $1-\delta$, we have
	\begin{equation*}
		P(\wh_m)-P_*
		\leq O\bigg(\frac{\alpha G^2\log (n)(\log n+\log(\frac{\log n}{\sqrt{\delta}}))^2}{n}\bigg)^{\frac{1}{2-\theta}}.
	\end{equation*}
\end{thm}
\begin{proof}
	We use the same notation as that in the proof of Theorem \ref{thm:high} unless specified.
	Define 
	\begin{equation}\label{new:a}
		a(n,\bar{\delta})=G\bigg(\frac{1}{\sqrt{n}}+\frac{17(1+\log n+\log(4\log n/\bar{\delta}))}{\sqrt{n}}\bigg).
	\end{equation}
	First we show that when $n\geq 100$, we have
	\begin{equation*}
		\frac{1}{2}\sqrt{\frac{2n}{\log_2 n}}\bigg(\frac{1}{\sqrt{n_0}}+\frac{17(1+\log n_0+\log(4\log n_0/\bar{\delta}))}{\sqrt{n_0}}\bigg)\geq 1.
	\end{equation*}
	Note that
	\begin{align*}
		&\text{LHS}\geq \sqrt{\frac{2n}{\log_2 n}}\bigg(\frac{\sqrt{17(1+\log n_0+\log(4\log n_0/\bar{\delta}))}}{\sqrt{n_0}}\bigg)\\
		&\geq \sqrt{\frac{34m(1+\log(\frac{n}{m}-1)+\log(4\log(\frac{n}{m}-1)/\bar{\delta}))}{\log_2 n}}\\
		&\geq\sqrt{\frac{17(\log_2 n-\log_2\log_2 n-3)\cdot\mathcal{F}_1}{\log_2 n}}\\
		&\geq \sqrt{17(1-\frac{\log_2\log_2 n+3}{\log_2 n})}\geq 1=\text{RHS},
	\end{align*}
	where $\mathcal{F}_1=(1+\log(\frac{n}{m}-1)+\log(2\log(\frac{n}{m}-1)\log_2 n/\delta))$. The first inequality holds by utilizing the fact that $a+b\geq 2\sqrt{ab}$, the second inequality holds since $n\geq 100$, and then $3\leq\frac{n}{m}-1\leq n_0=\lfloor \frac{n}{m}\rfloor\leq\frac{n}{m}$, the third inequality holds because of $m\geq\frac{1}{2}\log_2\frac{2n}{\log_2 n}-2>0$ and definition of $\bar{\delta}$, the fourth and fifth inequalities hold since $n\geq 100$ and $m\leq\frac{1}{2}\log_2 n$.
	
	We can duplicate the rest of the proof of Theorem \ref{thm:high} other than using the definition of $a(n_0,\bar{\delta})$ according to (\ref{new:a}). Finally, we have with probablity at least $1-\delta$,
	\begin{align*}
		&P(\wh_m)-P_*\leq (64\alpha)^{\frac{1}{2-\theta}}a(n_0,\bar{\delta})^{\frac{2}{2-\theta}}\\
		&\leq\left(\frac{64\alpha G^2(1+17\mathcal{F}_2)^2}{\frac{2n}{\log_2 n}-1} \right)^{\frac{1}{2-\theta}},
	\end{align*}
	where 
	\begin{align*}\mathcal{F}_2&=1+\log(\frac{n}{\frac{1}{2}\log_2\frac{2n}{\log_2 n}-2})\\
	&+\log(2\log(\frac{n}{\frac{1}{2}\log_2\frac{2n}{\log_2 n}-2})\log_2 n/\delta).
	\end{align*}
	The second inequality holds since $n_0=\lfloor \frac{n}{m}\rfloor\geq\frac{n}{m}-1$, $\frac{1}{2}\log_2\frac{2n}{\log_2 n}-2\leq m\leq\frac{1}{2}\log_2 n$.
\end{proof}
\subsection{A variant of ASA with a subroutine using proximal mapping}
In this section, we consider the nonsmooth composite optimization problem~(\ref{eqn:opt2}), which is
\begin{align*}
	\min_{\w\in\W} P(\w) \triangleq \E_{\z\sim \P}[f(\w, \z)] + r(\w).
\end{align*}
We introduce a variant of ASA, i.e., ASA3 (Algorithm \ref{alg:ASA3}), with a theoretical guarantee. ASA3 is a multistage scheme of proximal SGD (Algorithm \ref{alg:prox-SGD}).
		\begin{algorithm}[H]
			\caption{PSG$(\w_1,\gamma,T,\W)$}
			\label{alg:prox-SGD}
			\begin{algorithmic}[1]
				\REQUIRE~~$\w_1\in \mathcal{W}$, $\gamma>0$ and  $T$
				\ENSURE~~$\wh_T$
				\FOR {$t=1,\ldots, T$}
				\STATE Compute
				\begin{equation*}
				\begin{split} \w_{t+1}=\argmin\limits_{\w\in\W}&\frac{1}{2}\|\w-\w_t\|_2^2
				+\eta g_t^\top\w+\eta r(\w),
				\end{split}
				\end{equation*}
				where $g_t$ is the stochastic subgradient of $\E_{\z\sim \P}[f(\w, \z)]$ evaluated at $\w_t$
				\ENDFOR
				\STATE $\wh_T=\frac{1}{T}\sum_{t=1}^{T}{\w_t}$
				\RETURN $\wh_T$
			\end{algorithmic}
		\end{algorithm}
	%
		\begin{algorithm}[H]
			\caption{ASA3($\w_1,n,R_0$)}
			\label{alg:ASA3}
			\begin{algorithmic}[1]
				\REQUIRE~~$\w_1\in \W$, $n$ and $R_0=2R$
				\ENSURE~~$\wh_m$
				\STATE Set $\wh_0=\w_1$, $m=\lfloor \frac{1}{2}\log_2\frac{2n}{\log_2 n}\rfloor-1$, $n_0=\lfloor n/m\rfloor$ 
				\FOR{$k=1,\ldots,m$}
				\STATE Set $\gamma_k=\frac{R_{k-1}}{G\sqrt{n_0}}$ and $R_k = R_{k-1}/2$
				\STATE $$\wh_{k}=\text{PSG}(\wh_{k-1},\gamma_k,n_0,\W\cap\mathcal{B}(\wh_{k-1},R_{k-1}))$$
				\ENDFOR
				\RETURN $\wh_m$
			\end{algorithmic}
		\end{algorithm}

Before analysis, we first present a standard result of proximal SGD, which is the Lemma 5 of~\citep{DBLP:journals/corr/abs-1607-01027}.

\begin{prop}\label{prop:prox-sgd}
Suppose {\bf Assumptions~\ref{ass:1} and \ref{ass:2}} hold. In addition, we assume the proximal mapping in terms of $r(\w)$ has a closed form, and $r(\w)$ is $\rho$-Lipschitz continuous for any $\w\in\W$. Let $\epsilon\geq 0$ and $D$ be the upper bound of $\|\w_1-\w_{1,\epsilon}^{\dagger}\|_2$, where $\w_{1,\epsilon}^{\dagger}$ is the point closed to $\epsilon$-sublevel set of $P(\w)$. Denote $g_t$ by the stochastic subgradient of $\E_{\z\sim \P}[f(\w, \z)]$ at $\w_t$. Apply $T$-iterations of the following steps:
\begin{align*}
    \w_{t+1}=\argmin\limits_{\w\in\W\cap\mathcal{B}(\w_1,D)}\frac{1}{2}\|\w-\w_t\|_2^2+\eta g_t^\top\w+\eta r(\w).
\end{align*}
Given $\w_1$, for any $\delta\in(0,1)$, we have with probability at least $1-\delta$,
\begin{align*}
    &P(\widehat{\w}_T)-P(\w_{1,\epsilon}^{\dagger})
    \leq\frac{\eta G^2}{2}+\frac{\|\w_1-\w_{1,\epsilon}^{\dagger}\|_2^2}{2\eta T}+\frac{4GD\sqrt{3\log(1/\delta)}}{\sqrt{T}}+\frac{\rho D}{T},
\end{align*}
where $\widehat{\w}_T=\frac{1}{T}\sum_{t=1}^{T}\w_t$.
\end{prop}

\begin{thm}\label{thm:prox}
	Suppose {\bf Assumptions~\ref{ass:1} and \ref{ass:2}} hold. In addition, we assume the proximal mapping in terms of $r(\w)$ has a closed form, and $r(\w)$ is $\rho$-Lipschitz continuous for any $\w\in\W$. $\|\w_1 - \w^*\|_2\leq R_0$, where $\w^*$ is the closest optimal solution to $\w_1$. For  $n\geq 100$, $n_0\geq \frac{\rho^2}{G^2}$ and any $\delta\in(0,1)$, with probability at least $1-\delta$, the Algorithm ASA3 guarantees that 
	\begin{equation*}
		P(\wh_m)-P_*
		\leq O\bigg(\frac{\bar\alpha(\log (n)\log(\log(n)/\delta))}{n}\bigg)^{\frac{1}{2-\theta}}.
	\end{equation*}
	where $\bar\alpha=\max(\alpha G^2, (R_0G)^{2-\theta})$. 
\end{thm}
\begin{proof}
At first we derive the parallel version of the Proposition \ref{prop2} and Lemma \ref{lemma:new} in the case of solving problem (\ref{eqn:opt2}), which is not difficult by utilizing the Proposition \ref{prop:prox-sgd}. 
\begin{itemize}
\item We first prove the parallel version of the Proposition \ref{prop2}. By taking $\epsilon=0$, then $\w_{1,\epsilon}^{\dagger}$ is the projection of $\w_1$ onto the optimal set $\W_{*}$, and we define it to be $\w^*$. If $R_0$ is a upper bound of $\|\w_1-\w^*\|_2$, by taking $\eta=\frac{R_0}{G\sqrt{T}}$, then applying $T$ iterations of
\begin{align*}
    \w_{t+1}=\argmin\limits_{\w\in\W\cap\mathcal{B}(\w_1,R_0)}\frac{1}{2}\|\w-\w_t\|_2^2+\eta g_t^\top\w+\eta r(\w)
\end{align*}
has the guarantee that with probability at least $1-\delta$, 
\begin{align*}
P(\wh_T)-P_*\leq R_0G\left(\frac{1}{\sqrt{T}}+\frac{4\sqrt{3\log(1/\delta)}}{\sqrt{T}}\right)+\frac{\rho R_0}{T}.
\end{align*}
By choosing $T\geq\frac{\rho^2}{G^2}$, i.e., $\frac{\rho R_0}{T}\leq\frac{R_0 G}{\sqrt{T}}$, and we have
\begin{align*}
P(\wh_T)-P_*\leq R_0G\left(\frac{2}{\sqrt{T}}+\frac{4\sqrt{3\log(1/\delta)}}{\sqrt{T}}\right).
\end{align*}
\item We then prove the parallel version of the Lemma \ref{lemma:new}. We choose $\epsilon$ large enough such that $\w_{1,\epsilon}^{\dagger}=\w_1$. By utilizing the Proposition \ref{prop:prox-sgd}, we know that for any nonnegative $R_0$, taking $\eta=\frac{R_0}{G\sqrt{T}}$ and applying $T$ iterations of
\begin{align*}
    \w_{t+1}=\argmin\limits_{\w\in\W\cap\mathcal{B}(\w_1,R_0)}\frac{1}{2}\|\w-\w_t\|_2^2+\eta g_t^\top\w+\eta r(\w)
\end{align*}
have the guarantee that with probability at least $1-\delta$, 
\begin{align*}
P(\wh_T)-P(\w_1)\leq R_0G\left(\frac{1}{\sqrt{T}}+\frac{4\sqrt{3\log(1/\delta)}}{\sqrt{T}}\right)+\frac{\rho R_0}{T}.
\end{align*}
By choosing $T\geq\frac{\rho^2}{G^2}$, i.e., $\frac{\rho R_0}{T}\leq\frac{R_0 G}{\sqrt{T}}$, and we have
\begin{align*}
P(\wh_T)-P_*\leq R_0G\left(\frac{2}{\sqrt{T}}+\frac{4\sqrt{3\log(1/\delta)}}{\sqrt{T}}\right).
\end{align*}
\end{itemize}
The rest of the proof is similar to the proof of Theorem \ref{thm:high}.
\end{proof}

Finally, we mention that a stochastic mirror descent algorithm with  a non-Euclidean norm prox-function can be used, e.g., the Composite Objective Mirror Descent algorithm with $p$-norm divergence in~\citep{conf/colt/DuchiSST10}, Similar analysis based on Theorem 8 in~\citep{conf/colt/DuchiSST10} can be derived.  When leveraging the error bound, we can use a $p$-norm version (i.e., changing the Euclidean norm to the $p$-norm and the corresponding parameter $\alpha$). 
\end{document}